\documentclass[11pt]{article}
\usepackage[numbers, compress, sort]{natbib}
\usepackage[margin=1.1in]{geometry}

\usepackage{float}
\usepackage[utf8]{inputenc} 
\usepackage[T1]{fontenc}    
\usepackage{lmodern}
\usepackage{hyperref}       
\usepackage{url}            
\usepackage{booktabs}       
\usepackage{amsfonts}       
\usepackage{nicefrac}       
\usepackage{microtype}      
\usepackage[normalem]{ulem}
\usepackage{thm-restate}
\newcommand{\restatehack}[1]{}   
\usepackage{array}
\newcolumntype{R}{>{$}r<{$}} 
\newcolumntype{C}{>{$}c<{$}}

\usepackage{mathtools, amsmath, amssymb, graphicx, verbatim, amsthm}
\usepackage{color}
\definecolor{darkblue}{rgb}{0.0,0.0,0.2}
\definecolor{darkgreen}{rgb}{0.0,0.3,0.0}
\hypersetup{colorlinks,breaklinks,
            linkcolor=darkblue,urlcolor=darkblue,
            anchorcolor=darkblue,citecolor=darkblue}
\usepackage{wrapfig}
\usepackage[font=small]{caption}
\usepackage{subcaption}
\usepackage[colorinlistoftodos,textsize=tiny]{todonotes} 
\usepackage{bbm}
\usepackage{xspace}

\usetikzlibrary{calc}
\newcommand{\Comments}{1}
\newcommand{\mynote}[2]{\ifnum\Comments=1\textcolor{#1}{#2}\fi}
\newcommand{\mytodo}[2]{\ifnum\Comments=1
  \todo[linecolor=#1!80!black,backgroundcolor=#1,bordercolor=#1!80!black]{#2}\fi}

\ifnum\Comments=1               
\fi

\newcommand{\reals}{\mathbb{R}}

\newcommand{\dom}{\mathrm{dom}}

\newcommand{\prop}[1]{\mathrm{prop}[#1]}

\newcommand{\affhull}{\mathrm{affhull}}

\newcommand{\cell}{\mathrm{cell}}

\newcommand{\mode}{\mathrm{mode}}

\newcommand{\simplex}{\Delta_\Y}

\newcommand{\F}{\mathcal{F}}
\renewcommand{\H}{\mathcal{H}}

\newcommand{\R}{\mathcal{R}}
\newcommand{\Sc}{\mathcal{S}}

\newcommand{\U}{\mathcal{U}}
\newcommand{\V}{\mathcal{V}}
\newcommand{\X}{\mathcal{X}}
\newcommand{\Y}{\mathcal{Y}}

\newcommand{\risk}[1]{\underline{#1}}
\newcommand{\inprod}[2]{\langle #1, #2 \rangle}

\newcommand{\inter}{\mathrm{inter}}

\newcommand{\toto}{\rightrightarrows}

\newcommand{\trim}{\mathrm{trim}}
\newcommand{\red}{\mathrm{red}}
\newcommand{\trimred}{\mathrm{trim}}
\newcommand{\trimcover}{\mathrm{trim}}

\newcommand{\hyp}{\mathrm{hypo}}

\newcommand{\ones}{\mathbbm{1}}
\DeclarePairedDelimiter\ceil{\lceil}{\rceil}

\newcommand{\regret}[3]{R_{#1}(#2,#3)}

\newcommand{\Ind}[1]{\ones\{#1\}}

\newcommand{\hinge}{L_{\mathrm{hinge}}}
\newcommand{\ellzo}{\ell_{\text{0-1}}}
\newcommand{\ellabs}[1]{\ell_{#1}}

\newcommand{\emb}{{\tt e}}

\DeclareMathOperator*{\argmax}{arg\,max}
\DeclareMathOperator*{\argmin}{arg\,min}
\DeclareMathOperator*{\sgn}{sgn}

\newtheorem{theorem}{Theorem}
\newtheorem{lemma}{Lemma}
\newtheorem{proposition}{Proposition}
\newtheorem{corollary}{Corollary}

\newtheorem{claim}{Claim}

\newtheorem{definition}{Definition}
\newtheorem{construction}{Construction}

\newtheorem{assumption}{Assumption}

\title{An Embedding Framework for Consistent Polyhedral Surrogates}
\author{
 Jessie Finocchiaro \\
 \texttt{jefi8453@colorado.edu}
 \and
 Rafael Frongillo\\
 \texttt{raf@colorado.edu}
 \and
 Bo Waggoner\\
 \texttt{bwag@colorado.edu}
}
\date{
  University of Colorado Boulder
  \\[15pt]
  \today
}
\begin{document}

\maketitle

\begin{abstract}
We formalize and study the natural approach of designing convex surrogate loss functions via embeddings, for problems such as classification, ranking, or structured prediction. 
In this approach, one embeds each of the finitely many predictions (e.g.\ rankings) as a point in $\reals^d$, assigns the original loss values to these points, and ``convexifies'' the loss in some way to obtain a surrogate.
We establish a strong connection between this approach and polyhedral (piecewise-linear convex) surrogate losses.
Given any polyhedral loss $L$, we give a construction of a link function through which $L$ is a consistent surrogate for the loss it embeds.
Conversely, we show how to construct a consistent polyhedral surrogate for any given discrete loss.
Our framework yields succinct proofs of consistency or inconsistency of various polyhedral surrogates in the literature, and for inconsistent surrogates, it further reveals the discrete losses for which these surrogates are consistent.
We show some additional structure of embeddings, such as the equivalence of embedding and matching Bayes risks, and the equivalence of various notions of non-redudancy.
Using these results, we establish that indirect elicitation, a necessary condition for consistency, is also sufficient when working with polyhedral surrogates.
\end{abstract}

\section{Introduction}\label{sec:intro}

In supervised learning, one tries to learn a hypothesis which fits some labeled data, as judged by a target loss function.
Unfortunately, minimizing the target loss directly is typically computationally intractable, especially for discrete prediction tasks like classification, ranking, and structured prediction.
Instead, one typically minimizes a surrogate loss which is convex and therefore efficiently minimized.
Given a surrogate hypothesis, a link function then translates back to the target problem.
This general approach, called surrogate risk minimization, is ubiquitous in supervised machine learning algorithms.

A growing body of work seeks to design and analyze convex surrogates for particular target loss functions, and more broadly, understand the best empirical risk minimization bounds that can be found for a surrogate, for which consistency is a necessary condition.
For example, recent work has developed tools to bound the prediction dimension of the surrogate, meaning the dimension of the range of the surrogate hypothesis~\cite{frongillo2015elicitation,  ramaswamy2016convex}.
Yet in some cases these bounds are far from tight, such as for \emph{abstain loss} (classification with an abstain option)~\citep{bartlett2008classification,yuan2010classification,ramaswamy2016convex,ramaswamy2018consistent,zhang2018reject}.
Furthermore, the kinds of strategies available for constructing surrogates, and their relative power, are not well understood.

We augment this literature by studying a particularly natural approach for finding convex surrogates, wherein one ``embeds'' a discrete loss.
Specifically, we say a convex surrogate $L$ embeds a discrete loss $\ell$ if there is an injective embedding from the discrete reports (predictions) to a vector space such that (i) the original loss values are recovered, and (ii) a report is $\ell$-optimal if and only if the embedded report is $L$-optimal.
If this embedding can be extended to a calibrated link function, which roughly maps approximately $L$-optimal reports to $\ell$-optimal reports, then consistency follows~\citep{agarwal2015consistent}.
Common examples of this general construction include hinge loss as a surrogate for 0-1 loss and the abstain surrogate mentioned above~\citep{ramaswamy2018consistent}.

We prove that such an embedding scheme is intimately related to the class of polyhedral (piecewise-linear and convex) loss functions.
In particular, every discrete loss is embedded by a polyhedral surrogate.
Moreover, such an embedding gives rise to calibrated link function, and is therefore consistent with respect to the target loss.
Our proofs give explicit constructions for the surrogate (\S~\ref{sec:poly-loss-embed}) and link (\S~\ref{sec:calibration}) embedding a given discrete loss.

\restatehack{
  \begin{theorem}
    \label{thm:embed-poly-main}
    \label{thm:link-main}
  \end{theorem}}

\begin{restatable}{theorem}{embedpolyinformal}\label{thm:embed-poly-main}
  Every discrete loss $\ell$ is embedded by some polyhedral loss $L$, and every polyhedral loss $L$ embeds some discrete loss $\ell$.
\end{restatable}
\begin{restatable}{theorem}{linkinformal}\label{thm:link-main}
  Given any polyhedral loss $L$, let $\ell$ be a discrete loss it embeds. There exists a link function $\psi$ such that $(L,\psi)$ is calibrated with respect to $\ell$.
\end{restatable}

To better understand existing polyhedral surrogates, we provide tools to find the discrete losses they embed (Proposition~\ref{prop:representative-embeds-restriction}).
In short, if one can identify a finite \emph{representative set} $\Sc$ of reports for a surrogate $L$, meaning $\Sc$ always contains an $L$-optimal report for any label distribution, then $L$ embeds $L|_\Sc$, the loss given by $L$ restricting to $\Sc$.

Underpinning our results are several observations which formalize the idea that polyhedral losses ``behave like'' discrete losses.
For example, discrete losses have polyhedral Bayes risks (as the minimum of finitely many linear functions), as do polyhedral losses (Lemma~\ref{lem:polyhedral-range-gamma}).
As a consequence, polyhedral losses always have finite representative sets, and restricting the loss to any such set is an embedding.

We also provide several observations beyond what is needed to prove our main results, which we view as conceptual contributions (\S~\ref{sec:min-rep-sets},~\ref{sec:poly-ie-consistency}).
Using tools from property elicitation, we show an equivalence between minumum reprosentative sets and ``non-redundancy'', wherein no report is dominated by another.
We further show that, while the minimum representative set is not always unique, the loss values associated with it are unique, giving rise to a natural ``trim'' operation on losses.
Finally, using our main results, we show the following result: when restricting to the class of polyhedral surrogates, indirect elicitation is both necessary and sufficient for consistency (Theorem~\ref{thm:poly-ie-implies-consistent}).

Taken together, we view our contribution as both conceptual and practical.
We uncover the remarkable structure of polyhedral surrogates, deepening our understanding of the relationship between surrogate and discrete target losses.
This structure leads to a powerful new framework to design and analyze surrogate losses, which we apply to several examples.
We hope our framework will inspire new research, and we conclude with several exciting directions for future work.

\paragraph{Related works.}
The literature on convex surrogates focuses mainly on smooth surrogate losses~\citep{crammer2001algorithmic,bartlett2006convexity,bartlett2008classification, duchi2018multiclass, williamson2016composite, reid2010composite,menon2019multilabel,zhang2020convex,bao2020calibrated}.
Nevertheless, nonsmooth losses, such as the polyhedral losses we consider, have been proposed and studied for a variety of classification-like problems~\citep{yang2018consistency,yu2018lovasz,lapin2015top}.
Moreover,~\citet{zhang2020bayes} describe the impact of the hypothesis class has on consistency, and when consistency relative to the hypothesis class differs from Bayes consistency; the latter is what we describe in this paper when we say ``consistency.''

\citet{ramaswamy2018consistent} offer a notable addition to this literature is, arguing that nonsmooth losses may enable dimension reduction of the prediction space (range of the surrogate hypothesis) relative to smooth losses (cf.~\citep[Section 1.2]{ramaswamy2018consistent}).
They illustrate this phenomenon with a surrogate for \emph{abstain loss} needing only $\log(n)$ dimensions for $n$ labels, whereas the best known smooth loss needs $n-1$ dimensions.
Their surrogate is a natural example of an embedding (cf.~\S~\ref{sec:abstain}), and serves as inspiration for our work.

While property elicitation has by now an extensive literature~\citep{savage1971elicitation,osband1985information-eliciting,lambert2008eliciting,gneiting2011making,steinwart2014elicitation,frongillo2015vector-valued,fissler2016higher,lambert2018elicitation}, these works are mostly concerned with point estimation problems.
Literature directly connecting property elicitation to consistency is sparse.
However,~\citet{agarwal2015consistent} consider single-valued properties in finite outcome settings, whereas finite properties elicited by general convex losses are necessarily set-valued.
\citet{finocchiaro2021unifying} additionally relates indirect property elicitation to consistency when one is given either a target loss or property in both discrete and continuous prediction settings, assuming surrogates are minimizable, or attain their infimum in expectation over all distributions over the outcomes.

\section{Setting}
\label{sec:setting}

For discrete prediction problems like classification, the given discrete loss is often hard to optimize directly.
Therefore, many machine learning algorithms instead minimize a surrogate loss function with better optimization qualities, such as convexity.
To ensure that this surrogate loss successfully addresses the original problem, one needs to establish statistical consistency, a minimal requirement that is a prerequisite for generalization bounds.
Consistency depends crucially on the choice of link function that maps surrogate reports (predictions) to original reports.
The notion of \emph{calibration} (Definition~\ref{def:calibrated}) is equivalent to consistency in finite outcome settings~\citep{bartlett2006convexity,tewari2007consistency,ramaswamy2016convex} and depends solely on the conditional distribution over $\Y$.

\subsection{Notation and Losses}
\label{sec:notation-losses}

Let $\Y$ be a finite label space, and throughout let $n=|\Y|$.
Define $\reals^\Y_+$ to be the nonnegative orthant in $\reals^\Y$, i.e., $\reals^\Y_+ = \{x \in \reals^\Y \mid \forall y\in\Y\; x_y \geq 0 \}$.
Let $\simplex = \{p\in\reals^{\Y}_+ \mid \|p\|_1 = 1\}$ be the set of probability distributions on $\Y$, represented as vectors.
We will primarily focus on conditional distributions $p\in\simplex$ over labels, abstracting away the feature space $\X$; see \S~\ref{subsec:calibration-links} for a discussion of the joint distribution over $\X\times\Y$.

A generic loss function, denoted $L:\R\to\reals^\Y_+$, maps a report (prediction) $r$ from a set $\R$ to the vector of loss values $L(r) = (L(r)_y)_{y\in\Y}$ for each possible outcome $y\in\Y$.
We write the corresponding expected loss when $Y \sim p$ as $\inprod{p}{L(r)}$.
The \emph{Bayes risk} of a loss $L:\R\to\reals^\Y_+$ is the function $\risk{L}:\simplex\to\reals_+$ given by $\risk{L}(p) := \inf_{r\in\R} \inprod{p}{L(r)}$.
When restricting the domain of a loss $L$ from $\R$ to $\R'$, we write $L|_{\R'}$.

We assume that a given discrete prediction problem, such as classification, is given in the form of a discrete \emph{target loss} where $\R$ is a finite set.
We will denote target losses by $\ell:\R\to\reals^\Y_+$; when $\ell$ is written we assume $\R$ is a finite set.
Surrogate losses will take $\R = \reals^d$ and be written $L:\reals^d\to\reals^\Y_+$, typically with reports written $u\in\reals^d$.

For example, 0-1 loss is a discrete loss with $\R = \Y = \{-1,1\}$
given by $\ellzo(r)_y = \Ind{r \neq y}$, with Bayes risk $\risk{\ellzo}(p) = 1-\max_{y\in\Y} p_y$.
Two important surrogates for $\ellzo$ are hinge loss $\hinge(u)_y = (1-yu)_+$, where $(x)_+ = \max(x,0)$, and logistic loss $L(u)_y = \log(1+\exp(-yu))$ for $u\in\reals$.
See Figure~\ref{fig:bayes-risks-01} for a visualization of the Bayes risks of 0-1, Hinge, and Logistic losses, respectively.

Most of the surrogate losses we consider will be \emph{polyhedral}, meaning piecewise linear and convex; we therefore briefly recall the relevant definitions.
In $\reals^d$, a \emph{polyhedral set} or \emph{polyhedron} is the intersection of a finite number of closed halfspaces.
A \emph{polytope} is a bounded polyhedral set.
A convex function $f:\reals^d\to\reals$ is \emph{polyhedral} if its epigraph is polyhedral, or equivalently, if it can be written as a pointwise maximum of a finite set of affine functions~\citep{rockafellar1997convex}.
\begin{definition}[Polyhedral loss]
  A loss $L: \reals^d \to \reals^{\Y}_+$ is \emph{polyhedral} if $L(u)_y$ is a polyhedral convex function of $u$ for each $y\in\Y$.
\end{definition}
For example, hinge loss is polyhedral, whereas logistic loss is not.

\subsection{Property Elicitation}
\label{sec:property-elicitation}

To make headway, we will appeal to concepts and results from property elicitation.
This literature elevates the \emph{property}, or map from distributions to optimal reports, as a central object to study in its own right.
In our case, this map will often be set-valued, meaning a single distribution could yield multiple optimal reports.
(For example, when $p=(1/2,1/2)$, both $r=1$ and $r=-1$ optimize 0-1 loss.)
We will use double arrow notation to denote a (non-empty) set-valued map, so that $\Gamma: \simplex \toto \R$ is shorthand for $\Gamma: \simplex \to 2^{\R} \setminus \{\emptyset\}$.

\begin{definition}[Property, level set]\label{def:property}
  A \emph{property} is a function $\Gamma:\simplex\toto\R$.
  The \emph{level set} of $\Gamma$ for report $r$ is the set $\Gamma_r := \{p \in \simplex : r \in \Gamma(p)\}$.
\end{definition}

Intuitively, $\Gamma(p)$ is the set of reports which should be optimal for a given distribution $p$, and $\Gamma_r$ is the set of distributions for which the report $r$ should be optimal.
By optimal, we mean minimizing an associated loss function in expectation over $p$, which we formalize shortly.
Note that our definitions align such that discrete losses elicit finite properties (those with finite range). 
For example, the \emph{mode} is the 
property $\mode(p) = \argmax_{y\in\Y} p_y$, and captures the set of optimal reports for 0-1 loss: for each distribution over the labels, one should report the most likely label.
In this case we say 0-1 loss \emph{elicits} the mode, as we formalize below.

\begin{definition}[Elicits]
  \label{def:elicits}
  A loss $L:\R\to\reals^\Y_+$, \emph{elicits} a property $\Gamma:\simplex \toto \R$ if
  \begin{equation}
    \forall p\in\simplex,\;\;\;\Gamma(p) = \argmin_{r \in \R} \inprod{p}{L(r)}~.
  \end{equation}
  If $L$ elicits a property, it is unique and we denote it $\prop{L}$.
\end{definition}
Since we have defined a property $\Gamma$ to be nonempty, if the minimum of expected loss $\inprod{p}{L(\cdot)}$ is not attained for some $p \in \simplex$, then $L$ does not elicit a property.
We say that a loss $L$ is \emph{minimizable} if the infimum of $\inprod{p}{L(\cdot)}$ is attained for all $p \in \simplex$.

We will typically denote general properties and losses with $\Gamma$ and $L$, respectively.
For surrogate losses and properties, we will typically take $\R = \reals^d$.
For discrete target losses and properties, we will take $\R$ to be any finite set, and use lowercase notation $\gamma$ and $\ell$, respectively.
Any property $\gamma:\simplex\toto\R$ for a finite set $\R$ is called a \emph{finite property}.

\subsection{Calibration and Links}
\label{subsec:calibration-links}

To assess whether a surrogate and link function align with the original loss, we turn to the common condition of \emph{calibration}.
Roughly, a surrogate and link are calibrated if the best possible expected loss achieved by linking to an incorrect report is strictly suboptimal, which requires that the excess loss of some report is bounded by (a constant times) the excess loss of the linked report.

\begin{definition}
  \label{def:calibrated}
  Let discrete loss $\ell:\R\to\reals^\Y_+$, proposed surrogate $L:\reals^d\to\reals^\Y_+$, and link function $\psi:\reals^d\to\R$ be given.
  We say $(L,\psi)$ is \emph{calibrated} with respect to $\ell$ if
for all $p \in \simplex$,
  \begin{equation}
    \label{eq:calibrated}
  \inf_{u \in \reals^d : \psi(u) \not\in \gamma(p)} \inprod{p}{L(u)} > \inf_{u \in \reals^d} \inprod{p}{L(u)}~.
  \end{equation}
  If $(L, \psi)$ is calibrated with respect to $\ell$, we call $\psi$ a \emph{calibrated link.}
\end{definition}
It is well-known in finite-outcome settings that calibration is equivalent to \emph{consistency}, in the following sense (cf.~\citep{bartlett2006convexity,zhang2004statistical,agarwal2015consistent}).
Suppose we have the feature space $\X$ and label space $\Y$.
For any data distribution $D \in \Delta(\X \times \Y)$, let $L^*$ be the best possible expected $L$-loss achieved by any hypothesis $H:\X\to\reals^d$, and $\ell^*$ the best expected $\ell$-loss for any hypothesis $h:\X\to\R$, respectively.
We say $(L,\psi)$ is consistent with respect to $\ell$ if, for all data distributions $D \in \Delta(\X \times \Y)$, and all sequences of surrogate hypotheses $H_1,H_2,\ldots$ whose $L$-loss limits to $L^*$, the $\ell$-loss of the sequence $\psi\circ H_1,\psi \circ H_2, \ldots$ limits to $\ell^*$.

As Definition~\ref{def:calibrated} does not involve the feature space $\X$, we will drop it for the remainder of the paper.
Note that in the finite-outcome setting, calibration is necessary and sufficient for consistency from a generalization of~\citet{tewari2007consistency} given by~\citet{ramaswamy2016convex}.

\subsection{Embedding}

We now formalize the sense in which a convex surrogate can \emph{embed} a target loss $\ell$.
Here one maps each report (prediction) of $\ell$ to a point in $\reals^d$, then constructs a convex loss on $\reals^d$ that agrees with $\ell$ at these points.
This approach captures several consistent surrogates in the literature (e.g.,~\citep{ramaswamy2015hierarchical,ramaswamy2016convex,lapin2015top,wang2020weston}).

An important subtlety is that it is not always necessary to map \emph{all} target reports to $\reals^d$.
It is often convenient to allow $\ell$ to have reports that are ``redundant'' in some sense. (We explore redundancy further in \S~\ref{sec:min-rep-sets}; see also \citet{wang2020weston}.)
Because of this redundancy, we will only require an embedding map to be defined on a \emph{representative set}: a set of reports $\Sc$ such that, for all label distributions, at least one report $r\in\Sc$ minimizes expected loss.
\begin{definition}[Representative set]
  Let $\Gamma:\simplex\toto\R$.
  We say $\Sc \subseteq \R$ is \emph{representative for $\Gamma$} if we have $\Gamma(p) \cap \Sc \neq \emptyset$ for all $p\in \simplex$.
  We further say $\Sc$ is a \emph{minimum representative set} if it has the smallest cardinality among all representative sets.
  Given a minimizable loss $L:\R\to\reals^\Y_+$, we say $\Sc$ is a (minimum) representative set for $L$ if it is a (minimum) representative set for $\prop L$.
\end{definition}

\citet{wang2020weston} first studies the notion of minimum representative sets under the name \emph{embedding cardinality}.

We now define an embedding.
In addition to matching loss values, as described above, we require the original reports to be optimal exactly when the corresponding embedded points are optimal.
\begin{definition}[Embedding]\label{def:loss-embed}
  A minimizable loss $L:\reals^d\to\reals^\Y_+$ \emph{embeds} a loss $\ell:\R\to\reals^\Y_+$ if there exists a representative set $\Sc$ for $\ell$ and an injective embedding $\varphi:\Sc\to\reals^d$ such that
  (i) for all $r\in\Sc$ we have $L(\varphi(r)) = \ell(r)$, and (ii) for all $p\in\simplex,r\in\Sc$ we have
  \begin{equation}\label{eq:embed-loss}
    r \in \prop{\ell}(p) \iff \varphi(r) \in \prop{L}(p)~.
  \end{equation}
  If $\Sc$ is a minimal representative set, we say $L$ \emph{tightly embeds} $\ell$.
\end{definition}

To illustrate the idea of embedding, let us examine hinge loss in detail as a surrogate for 0-1 loss for binary classification.
Recall that we have $\R = \Y = \{-1, +1\}$, with $\hinge(u)_y = (1 - uy)_+$ and $\ellzo(r)_y := \Ind{r\neq y}$, typically with link function $\psi(u) = \sgn(u)$.
We will see that hinge loss embeds (2 times) 0-1 loss, via the embedding $\varphi(r) = r$.
For condition (i), it is straightforward to check that $\hinge(r)_y = 2\ellzo(r)_y$ for all $r,y\in\{-1,1\}$.
For condition (ii), let us compute the property each loss elicits, i.e., the set of optimal reports for each $p\in\simplex$:
\[
\prop{\ellzo}(p) = \begin{cases}
1 & p_1 > 1/2 \\
\{-1,1\} & p_1 = 1/2\\
-1 & p_1 < 1/2
\end{cases}
\qquad
\prop{L_{hinge}}(p) = \begin{cases}
[1,\infty) & p_1 = 1\\
1 & p_1 \in (1/2,1) \\
[-1,1] & p_1 = 1/2\\
-1& p_1 \in (0, 1/2)\\
(-\infty, -1]& p_1 = 0
\end{cases}~.
\]
In particular, we see that $-1 \in \prop{\ellzo}(p) \iff p_1 \in [0, 1/2] \iff -1 \in \prop{\hinge}(p)$, and $1 \in \prop{\ellzo}(p) \iff p_1 \in [1/2,1] \iff 1 \in \prop{\hinge}(p)$.
With both conditions of Definition~\ref{def:loss-embed} satisfied, we can conclude that $\hinge$ embeds $2\ellzo$.
By results in \S~\ref{subsec:match-BR}, one could also show that $\hinge$ embeds $2\ellzo$ by the fact that their Bayes risks match (Figure~\ref{fig:bayes-risks-01}).

\begin{figure}
	\begin{minipage}{0.3\linewidth}
	\centering
	\includegraphics[width=0.95\linewidth]{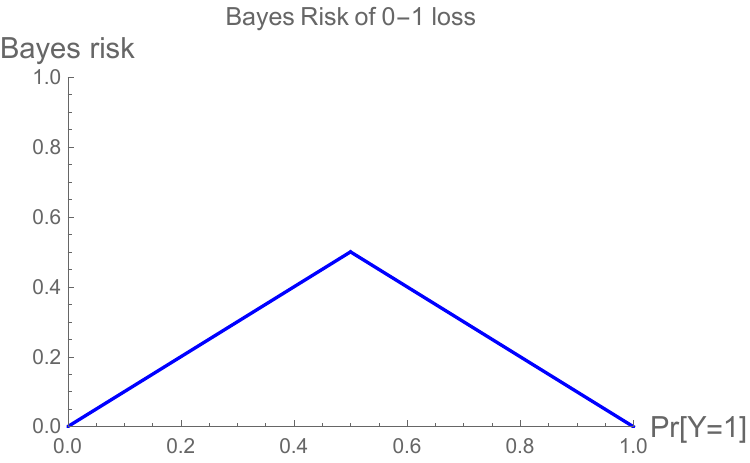}
	\end{minipage}
	\hfill
	\begin{minipage}{0.3\linewidth}
	\centering		\includegraphics[width=0.95\linewidth]{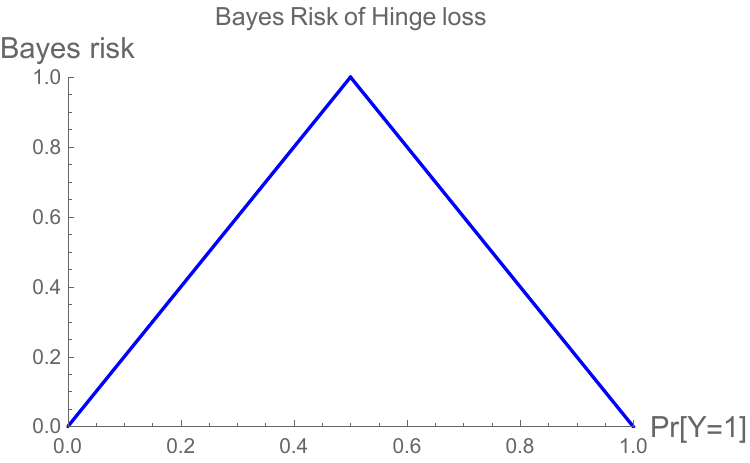}
	\end{minipage}
	\hfill
	\begin{minipage}{0.3\linewidth}
	\centering
	\includegraphics[width=0.95\linewidth]{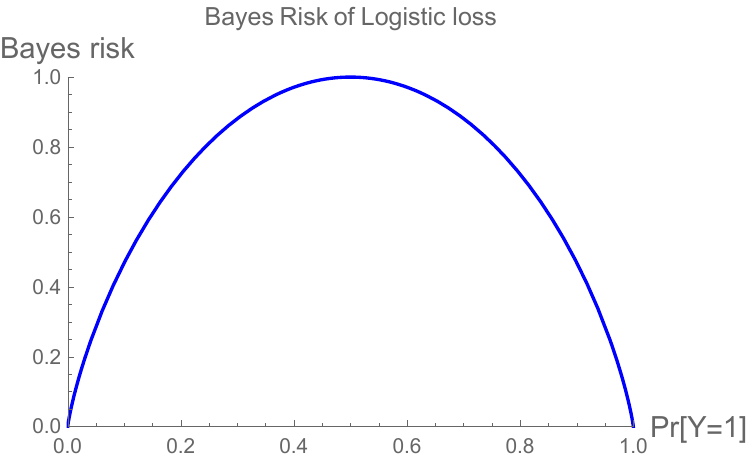}
\end{minipage}
\caption{Bayes risks $L : p \mapsto \inf_u \inprod{p}{L(u)}$ of 0-1, hinge, and logistic losses, respectively, plotted as a function of $p_1 = Pr[Y=1]$.
	Observe that the Bayes risks of 0-1 and hinge loss are both piecewise lienar and concave, while the Bayes risk of logistic loss is also concave, but not piecewise linear.  
	In particular, the Bayes risk of Hinge loss is twice that of 0-1.  We show in Proposition~\ref{prop:embed-bayes-risks} that embedding is equivalent to matching Bayes risks. 
}
\label{fig:bayes-risks-01}
\end{figure}

In this particular example, it is known $(\hinge,\psi)$ is calibrated for $\psi(u) = \sgn(u)$.
More generally, however, it is not clear whether an arbitrary embedding yields a calibrated link.
Indeed, apart from mapping the embedded points back to their original reports, via $\psi(\varphi(r)) = r$, how to map the remaining values is far from obvious.
When the surrogate is polyhedral, we give a construction to map the remaining values in \S~\ref{sec:calibration}, showing that embeddings always yield calibration.
We first explore in \S~\ref{sec:poly-loss-embed} the connection between embeddings and polyhedral surrogates.

While our notion of embedding is sufficient for calibration (and therefore consistency), it is worth noting that it is not \emph{necessary} for these conditions.  
For example, while logistic loss does not embed 0-1 loss, the surrogate and link for logistic loss are consistent.

\section{Embeddings and Polyhedral Losses}
\label{sec:poly-loss-embed}

In this section, we establish a tight relationship between the technique of embedding and the use of polyhedral (piecewise-linear convex) surrogate losses, showing Theorem~\ref{thm:embed-poly-main}.
We defer the question of when such surrogates are consistent to \S~\ref{sec:calibration}. 

A first observation is that if a loss $L$ elicits a property $\Gamma$, then $L$ restricted to some representative set $\Sc$, denoted $L|_\Sc$, elicits $\Gamma$ restricted to $\Sc$.
As a consequence, restricting to representative sets preserves the Bayes risk.
We will use these observations throughout.
\begin{lemma}\label{lem:loss-restrict}
  Let $L:\R\to\reals^\Y_+$ elicit $\Gamma$, and let $\Sc\subseteq\R$ be representative for $L$.
  Then $L|_\Sc$ elicits $\gamma:\simplex\toto\Sc$ defined by $\gamma(p) = \Gamma(p)\cap \Sc$.
  Moreover, $\risk{L}=\risk{L|_\Sc}$.
\end{lemma}
\begin{proof}
  Let $p\in\simplex$ be fixed throughout.
  First let $r \in \gamma(p) = \Gamma(p) \cap \Sc$.
  Then $r \in \Gamma(p) = \argmin_{u\in\R} \inprod{p}{L(u)}$, so as $r\in\Sc$ we have in particular $r \in \argmin_{u\in\Sc} \inprod{p}{L(u)}$.
  For the other direction, suppose $r \in \argmin_{u\in\Sc} \inprod{p}{L(u)}$.
  As $\Sc$ is representative for $L$, we must have some $s \in \Gamma(p) \cap \Sc$.
  On the one hand, $s\in\Gamma(p) = \argmin_{u\in\R} \inprod{p}{L(u)}$.
  On the other, as $s \in \Sc$, we certainly have $s \in \argmin_{u\in\Sc} \inprod{p}{L(u)}$.
  But now we must have $\inprod{p}{L(r)} = \inprod{p}{L(s)}$, and thus $r \in \argmin_{u\in\R} \inprod{p}{L(u)} = \Gamma(p)$ as well.
  We now see $r \in \Gamma(p) \cap \Sc$.
  Finally, the equality of the Bayes risks $\min_{u\in\R} \inprod{p}{L(u)} = \min_{u\in\Sc} \inprod{p}{L(u)}$ follows immediately by the above, as $\emptyset \neq \Gamma(p)\cap\Sc \subseteq \Gamma(p)$ for all $p\in\simplex$.
\end{proof}

Lemma~\ref{lem:loss-restrict} leads to the following useful tool for finding embeddings.
\begin{proposition}\label{prop:representative-embeds-restriction}
  Let a minimizable surrogate loss $L:\reals^d \to \reals^\Y_+$ be given.
  If $L$ has a finite representative set $\U \subseteq \reals^d$, then $L$ embeds the discrete loss $L|_\U$.
\end{proposition}
\begin{proof}
  Let $\Gamma = \prop{L}$ and $\gamma = \prop{L|_\U}$.
  Define $\varphi : \U \to \U$ to be the identity embedding.
  Condition (i) of an embedding is trivially satisfied, as $L|_\U(u) = L(u)$ for all $u\in\U$.
  Now let $u\in\U$.
  From Lemma~\ref{lem:loss-restrict}, for all $p\in\simplex$ we have $u \in \gamma(p) \iff u \in \Gamma(p) \cap \U \iff u \in \Gamma(p)$.
  We conclude condition (ii) of an embedding.
\end{proof}

We now shift our focus to \emph{polyhedral} (piecewise-linear and convex) surrogates.
Our first observation is that while polyhedral surrogates cannot elicit finite properties, in the sense that they have infinitely many possible reports, they do elicit properties with a finite range, meaning a finite set of possible optimal sets.
This observation lets us apply results about finite representative sets to understand the structure of polyhedral surrogates and the losses they embed.
See \S~\ref{app:power-diagrams} for the full proof.

\begin{restatable}{lemma}{polyhedralrangegamma}
	\label{lem:polyhedral-range-gamma}
	Let $L:\reals^d\to\reals_+^\Y$ be a polyhedral loss; then $L$ is minimizable and elicits a property $\Gamma := \prop{L}$.
	Then the range of $\Gamma$, given by $\Gamma(\simplex) = \{\Gamma(p) \subseteq \reals^d : p\in\simplex\}$, is a finite set of closed polyhedra.
\end{restatable}
\begin{proof}[Sketch]
	We know that $L$ is minimizable from~\citet[Corollary 19.3.1]{rockafellar1997convex} as $L$ is bounded from below.
	With $\Y$ finite, there are only finitely many supporting sets over $\simplex$.
	For $p \in \simplex$, the power diagram induced by projecting the epigraph of expected loss onto $\reals^d$ is the same for any $p$ of the same support (Lemma~\ref{lem:polyhedral-pd-same}).
	Moreover, we have $\Gamma(p)$ being exactly one of the faces of the projected epigraph since the hyperplane $u \mapsto (u, \inprod{p}{L(u)})$ supports the epigraph of the expected loss at exactly the property value; moreover, since the loss is polyhedral the supporting hyperplane must support on a face of the epigraph.
	Since this epigraph has finitely many faces (as it is polyhedral), the range of $\Gamma$ is then (a subset) of elements of a finitely generated (finite supports) set of finite elements (finite faces).
	Moreover, each element of $\Gamma(\simplex)$ is a closed polyhedron since it corresponds exactly to a closed face of a polyhedral set.
\end{proof}

\begin{theorem}\label{thm:poly-embeds-discrete}
  Every polyhedral loss $L$ embeds a discrete loss.
\end{theorem}
\begin{proof}
  Let $L:\reals^d\to\reals_+^\Y$ be a polyhedral loss, and $\Gamma = \prop{L}$.
  By Lemma~\ref{lem:polyhedral-range-gamma}, $\Gamma(\simplex)$ is finite set. 
  For each $U\in \Gamma(\simplex)$, select $u_U \in U$, and let $\Sc = \{u_U : U \in\Gamma(\simplex)\}$, which is again finite.
  For any $p\in\simplex$ then, let $U = \Gamma(p)$.
  We have $U \in \Gamma(\simplex)$ by definition, and thus some $u_U \in \Sc$; in particular, $u_U \in U = \Gamma(p)$.
  We conclude that $\Sc$ is representative for $L$.
  Proposition~\ref{prop:representative-embeds-restriction} now states that $L$ embeds $L|_\Sc$.
\end{proof}

We now turn to the reverse direction: which discrete losses are embedded by some polyhedral loss?
Perhaps surprisingly, we show in Theorem~\ref{thm:discrete-loss-poly-embeddable} that \emph{every} discrete loss is embeddable.
Combining this result with Theorem~\ref{thm:poly-embeds-discrete} establishes Theorem~\ref{thm:embed-poly-main}.
Further combining with Theorem~\ref{thm:link-main}, proved in the following section, this construction gives a consistent polyhedral surrogate for every discrete target loss.

The proof of Theorem~\ref{thm:discrete-loss-poly-embeddable} uses a construction via convex conjugate duality which has appeared in several different forms in the literature (e.g.\ \cite{duchi2018multiclass,abernethy2013efficient,frongillo2014general}).
We then apply a result we will prove in \S~\ref{sec:min-rep-sets}: a minimizable surrogate embeds a discrete loss if and only if their Bayes risks match (Proposition~\ref{prop:embed-bayes-risks}).

\begin{theorem}\label{thm:discrete-loss-poly-embeddable}
  Every discrete loss $\ell:\R \to \reals^\Y_+$ is embedded by a polyhedral loss.
\end{theorem}
\begin{proof}
  Let $n = |\Y|$, and let $C:\reals^n \to \reals$ be given by $(-\risk{\ell})^*$, the convex conjugate of $-\risk{\ell}$.
  From standard results in convex analysis, $C$ is polyhedral as $-\risk{\ell}$ is, and $C$ is finite on all of $\reals^\Y$ as the domain of $-\risk{\ell}$ is bounded~\cite[Corollary 13.3.1]{rockafellar1997convex}.
  Note that $-\risk{\ell}$ is a closed convex function, as the infimum of affine functions, and thus $(-\risk{\ell})^{**} = -\risk{\ell}$.
  Define $L:\reals^n\to\reals^\Y$ by $L(u) = C(u)\ones - u$, where $\ones\in\reals^\Y$ is the all-ones vector.
  As $C$ is polyhedral, so is $L$.
  We first show that $L$ embeds $\ell$, and then establish that the range of $L$ is in fact $\reals^\Y_+$, as desired.

  We compute Bayes risks and apply Proposition~\ref{prop:embed-bayes-risks} to see that $L$ embeds $\ell$.
  Observe that $\risk{\ell}$ is polyhedral as $\ell$ is discrete.
  For any $p\in\simplex$, we have
  \begin{align*}
    \risk{L}(p)
    &= \inf_{u\in\reals^n} \inprod{p}{C(u)\ones - u}\\
    &= \inf_{u\in\reals^n} C(u) - \inprod{p}{u}\\
    &= -\sup_{u\in\reals^n} \inprod{p}{u} - C(u)\\
    &= -C^*(p) = - (-\risk{\ell}(p))^{**} = \risk{\ell}(p)~.
  \end{align*}
  It remains to show $L(u)_y \geq 0$ for all $u\in\reals^n$, $y\in\Y$.
  Letting $\delta_y\in\simplex$ be the point distribution on outcome $y\in\Y$, we have for all $u\in\reals^n$, $L(u)_y \geq \inf_{u'\in\reals^n} L(u')_y = \risk{L}(\delta_y) = \risk{\ell}(\delta_y) \geq 0$, where the final inequality follows from the nonnegativity of $\ell$.
\end{proof}

While Theorem~\ref{thm:discrete-loss-poly-embeddable} constructs a consistent surrogate for any discrete loss, in some settings, such as structured prediction and information retrieval, the prediction dimension $d = n := |\Y|$ can be prohibitively large.
\footnote{One can always reduce to $d=n-1$ in Theorem~\ref{thm:discrete-loss-poly-embeddable} via a linear transformation from $\reals^n$ to $\reals^{n-1}$ which is injective on $\simplex$; redefining the surrogate appropriately, the Bayes risks will still match.}
Recent work~\citep{ramaswamy2016convex,finocchiaro2020embedding,finocchiaro2021unifying} yield characterizations for bounding the prediction dimension $d$ for consistent convex surrogates and embeddings.

\section{Consistency via Calibrated Links}
\label{sec:calibration}

We have now seen the tight relationship between polyhedral losses and embeddings; in particular, every polyhedral loss embeds some discrete loss.
The embedding itself tells us how to link the embedded points back to the discrete reports (map $\varphi(r)$ to $r$).
But it is not clear how to extend this to yield a full link function $\psi: \reals^d \to \R$, and whether such a $\psi$ can lead to consistency.
In this section, we prove Theorem~\ref{thm:link-main}, restated below, which gives a construction to generate calibrated links for \emph{any} polyhedral surrogate.

\linkinformal*

Theorem \ref{thm:link-main} will follow immediately from Theorems \ref{thm:calibrated-separated} and \ref{thm:thickened-separated}, as discussed below.
Their full proofs appear in Appendices \ref{sec:equiv-sep-calib} and \ref{app:sep-link-exists} respectively.

Theorem \ref{thm:calibrated-separated} shows that calibration is equivalent to a geometric condition, which we call \emph{separation}, of a link function $\psi$.
Recall that for indirect elicitation, any point $u \in \Gamma(p)$ must link to a report $\psi(u) \in \gamma(p)$.
(In terms of losses, $u$ minimizing expected $L$-loss implies that $\psi(u)$ minimizes expected $\ell$-loss, with respect to $p$.)
The idea of separation is that points in the neighborhood of $u$ must also link to to a report in $\gamma(p)$.
Furthermore, there must be a uniform lower bound $\epsilon$ on the size of any such neighborhood.

\begin{definition}[Separated Link]\label{def:sep-link}
  Let properties $\Gamma:\simplex\toto\reals^d$ and $\gamma:\simplex\toto\R$ be given.
  We say a link $\psi:\reals^d\to\R$
  is \emph{$\epsilon$-separated with respect to $\Gamma$ and $\gamma$} if for all $u\in\reals^d$ with $\psi(u)\notin\gamma(p)$, we have $d_\infty(u,\Gamma(p)) > \epsilon$, where $d_\infty(u,A) \doteq \inf_{a\in A} \|u-a\|_\infty$.
  Similarly, we say $\psi$ is $\epsilon$-separated with respect to $L$ and $\ell$ if it is $\epsilon$-separated with respect to $\prop{L}$ and $\prop{\ell}$.
\end{definition}

\begin{restatable}{theorem}{calibratedseparated} \label{thm:calibrated-separated}
  Let polyhedral surrogate $L:\reals^d \to \reals^\Y_+$, discrete loss $\ell:\R\to\reals^\Y_+$, and link $\psi:\reals^d\to\R$ be given.
  Then $(L,\psi)$ is calibrated with respect to $\ell$ if and only if
  $\psi$ is $\epsilon$-separated with respect to $L$ and $\ell$ for some
  $\epsilon>0$.
\end{restatable}

To prove Theorem \ref{thm:link-main}, it now suffices to show that for any polyhedral $L$ embedding some $\ell$, there exists a \emph{separated} link $\psi$ with respect to $L$ and $\ell$.
This is given by Construction ~\ref{const:eps-thick-link} below.

\begin{restatable}{theorem}{thickenedseparated} \label{thm:thickened-separated}
  Let polyhedral surrogate $L:\reals^d \to \reals^\Y_+$ embed the discrete loss $\ell:\R\to\reals^\Y_+$.
  Then there exists $\epsilon_0 > 0$ such that, for all $0 < \epsilon \leq \epsilon_0$, Construction~\ref{const:eps-thick-link} yields an $\epsilon$-separated link with respect to $L$ and $\ell$.
\end{restatable}

To set the stage for Construction \ref{const:eps-thick-link}, we sketch the two main steps in proving Theorem \ref{thm:thickened-separated}: \emph{(a)} showing that one can produce a link $\psi$ such that $(L,\psi)$ indirectly elicits $\ell$; \emph{(b)} ``thickening'' $\psi$ such that it is separated.

For \emph{(a)}, begin by linking each embedding point back to its original report.
Now we must determine $\psi(u)$ for non-embedding points.
The challenge is that we may have $u \in \Gamma(p) \cap \Gamma(p')$.
Because $u$ minimizes expected surrogate loss for both $p$ and $p'$, the link must satisfy $\psi(u) \in \gamma(p) \cap \gamma(p')$.
It is not even clear \emph{a priori} that these sets intersect.
We use the definition of embedding and elicitation results, discussed in \S~\ref{sec:min-rep-sets}, to show that for each such $u$ there exists $r \in \R$ such that $\Gamma_u \subseteq \gamma_r$, i.e. any $p$ satisfying $u \in \Gamma(p)$ also satisfies $r \in \gamma(p)$.
This implies that if $u \in \Gamma(p) \cap \Gamma(p')$, then there exists $r \in \gamma(p) \cap \gamma(p')$, so we may safely choose $\psi(u) = r$.

For \emph{(b)}, we show that this link can be ``thickened'' by some positive $\epsilon$, as described next.
Consider an optimal surrogate report set, i.e. set of the form $U = \Gamma(p) = \argmin_u \inprod{p}{L(u)}$.
By indirect elicitation, $\psi$ is already correct on $U$.
Now, we ``thicken'' $U$ to obtain $U_{\epsilon} = \{u : \|u - U\| \leq \epsilon\}$.
Then we require that all points in $U_{\epsilon}$ are linked to some element of $\gamma(p) = \argmin_r \inprod{p}{\ell(r)}$.
For $\epsilon > 0$, this directly implies separation.

However, it is not clear that this linking is possible because a point $u$ may be in multiple thickened sets $U_{\epsilon}, U'_{\epsilon}$, etc.
Therefore, we need to take each possible collection $U,U'$, etc. and thicken their intersection in an analogous way.

Given $u \in U \cap U' \cap \dots$, we use $\Psi(u)$ to denote the remaining legal choices for $\psi(u)$ after imposing the requirements for each such set $U,U'$, etc.
The key claim is that, for small enough $\epsilon > 0$, $\Psi(u)$ is nonempty: at least one legal value for $\psi(u)$ remains.
This claim follows from a geometric result (Lemma \ref{lemma:thick-empty}) that, for all small enough $\epsilon$, a subset of thickenings $U_{\epsilon}$ intersect if and only if the $U$ sets themselves intersect.
When they do intersect, indirect elicitation implies that there exists a legal choice of link for the intersection of the thickenings.
It is also important that, by Lemma~\ref{lem:polyhedral-range-gamma}, for polyhedral surrogates there are only finitely many sets of the form $U = \Gamma(p)$.
This yields a single uniform smallest $\epsilon$ such that the key claim is true for all $u \in \reals^d$.

Given the above proof sketch, the following construction is relatively straightforward.
We initialize the link using the embedding points and optimal report sets, then use $\Psi$ to narrow down to only legal choices; we then pick from $\psi(u)$ from $\Psi(u)$ arbitrarily.
Theorem \ref{thm:thickened-separated} implies that, for all small enough $\epsilon$, the resulting link $\psi$ is well-defined at all points.
\begin{construction}[$\epsilon$-thickened link] \label{const:eps-thick-link}
  Given a polyhedral $L$ that embeds some $\ell$, an $\epsilon > 0$, and a norm $\|\cdot\|$, the \emph{$\epsilon$-thickened link} $\psi$ is constructed as follows.
  First, define $\U = \{\Gamma(p) : p \in \simplex\}$.
  For each $U \in \U$, let $R_U = \{r \in \R : \varphi(r) \in U\}$, the reports whose embedding points are in $U$.
  First, initialize $\Psi: \reals^d \toto \R$ by setting $\Psi(u) = \R$ for all $u$.
  Then for each $U \in \U$, for all points $u$ such that $\inf_{u^* \in U} \|u^*-u\| < \epsilon$, update $\Psi(u) = \Psi(u) \cap R_U$.
  Finally, define $\psi(u) \in \Psi(u)$, breaking ties arbitrarily.
  If $\Psi(u)$ became empty, then leave $\psi(u)$ undefined.
\end{construction}

\paragraph{Remarks.}
Construction~\ref{const:eps-thick-link} is not necessarily computationally efficient as the number of labels $n$ grows.
In practice this potential inefficiency is not typically a concern, as the family of losses typically has some closed form expression in terms of $n$, and thus the construction can proceed at the symbolic level.
We illustrate this formulaic approach in \S~\ref{sec:abstain}.

Applying the $\epsilon$-thickened link construction additionally enables one to verify the consistency of a proposed link $\psi^*$.
For a given $\epsilon$ and norm $\|\cdot\|$, suppose one follows the routine of Construction~\ref{const:eps-thick-link} until the last step in which values for the link $\psi$ are selected.
Instead, we can simply test whether the proposed link values are contained in the valid choices, i.e., if $\psi^*(u) \in \Psi(u)$ for all $u\in\reals^d$.
If so, then the proposed link $\psi^*$ is calibrated.

\paragraph{Regret transfer rates of calibrated polyhedral surrogates.}
Recall that the goal of surrogate regret minimization is to learn a hypothesis $h$ that minimizes expected surrogate loss, then output hypothesis $\psi \circ h$, which hopefully minimizes expected target loss.
Consistency is a minimal requirement: when surrogate regret\footnote{Regret in this context is the difference between the expected loss of a hypothesis and the expected loss of the Bayes optimal hypothesis that minimizes expected loss. We refer the reader to \cite{frongillo2021surrogate} for a formal definition.} of $h$ converges to zero, i.e. $\text{Regret}_L(h) \to 0$, so does target regret of $\psi \circ h$, i.e. $\text{Regret}_{\ell}(\psi \circ h) \to 0$.
A natural question is whether \emph{fast} convergence in surrogate regret implies fast convergence in target regret.
A recent paper \cite{frongillo2021surrogate} shows that, for polyhedral surrogates, this is always the case.
\begin{theorem}[\cite{frongillo2021surrogate}, Theorem 1]
  Let $(L,\psi)$ be a polyhedral surrogate that is consistent for a discrete loss $\ell$.
  Then there exists $c > 0$ such that, for all hypotheses $h$, $\text{Regret}_{\ell}(\psi \circ h) \leq c \cdot \text{Regret}_L(h)$.
\end{theorem}

\section{Consistency of abstain surrogate and link construction}
\label{sec:abstain}

Several authors consider a variant of classification, with the addition of a reject or \emph{abstain} option~\citep{bartlett2008classification,ramaswamy2018consistent,madras2018predict,elyaniv2010foundations,cortes2016learning}.
In particular, \citet{ramaswamy2018consistent} study the loss $\ellabs{\alpha} : [n] \cup \{\bot\} \to \reals^\Y_+$ defined by $\ellabs{\alpha}(r)_y = 0$ if $r=y$, $\alpha$ if $r = \bot$, and 1 otherwise.
The report $\bot$ corresponds to ``abstaining'' if no label is sufficiently likely, specifically, if no $y\in\Y$ has $p_y \geq 1-\alpha$.
\citeauthor{ramaswamy2018consistent} provide a polyhedral surrogate $L_\alpha$ for $\ellabs{\alpha}$, which we present here for $\alpha=1/2$.
Letting $d = \ceil{\log_2(n)}$, their surrogate is $L_{1/2} : \reals^d \to \reals^\Y_+$ given by
\begin{equation}\label{eq:abstain-surrogate}
L_{1/2}(u)_y = \left(\max\nolimits_{j \in [d]}\varphi(y)_j u_j + 1\right)_+~,
\end{equation}
where $\varphi$ embeds outcomes to corners of the $\pm 1$ hypercube, and the abstain report $\bot$ to the origin.
Consistency is proven for the following link function,
\begin{equation}\label{eq:abstain-link}
  \psi(u) = \begin{cases}
	\bot & \min_{i \in [d]} |u_i| \leq 1/2\\
	\varphi^{-1}(\sgn(-u)) &\text{otherwise}
  \end{cases}~.
\end{equation}

As we illustrate in Figure~\ref{fig:abstain-links}(L), the link function $\psi$ proposed by \citeauthor{ramaswamy2018consistent} can be recovered from Theorem~\ref{thm:link-main} by choosing the norm $\|\cdot\|_\infty$ and setting $\epsilon=1/2$. 
Hence, our framework would have simplified the process of finding such a link, and the corresponding proof of consistency.
To illustrate this point further, we give an alternate link $\psi_1$ corresponding to $\|\cdot\|_1$ and $\epsilon=1$, shown in Figure~\ref{fig:abstain-links}(R):
\begin{equation}\label{eq:abstain-link-1}
  \psi_1(u) = \begin{cases}
	\bot & \|u\|_1 \leq 1\\
	\varphi^{-1}(\sgn(-u)) &\text{otherwise}
  \end{cases}~.
\end{equation}
Theorem~\ref{thm:link-main} immediately gives calibration of $(L_{1/2},\psi_1)$ with respect to $\ellabs{1/2}$.
Aside from its simplicity, one possible advantage of $\psi_1$ is that it assigns $\bot$ to much less of the surrogate space $\reals^d$.
It would be interesting to compare the two links in practice.

\begin{figure}
\begin{center}
\begin{minipage}{0.32\linewidth}
\includegraphics[width=\linewidth]{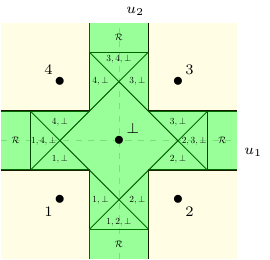}
\end{minipage}\hfill
\begin{minipage}{0.32\linewidth}
\includegraphics[width=\linewidth]{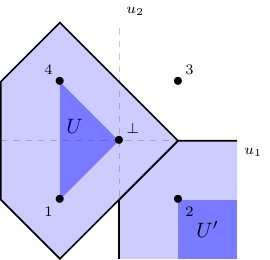}
\end{minipage}\hfill
\begin{minipage}{0.32\linewidth}
\includegraphics[width=\linewidth]{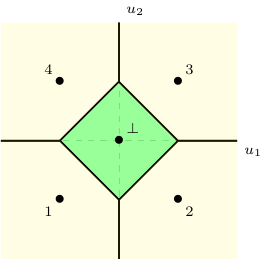}
\end{minipage}\hfill
\caption{Constructing links for the abstain surrogate $L_{1/2}$ with $d=2$. The embedding is shown in bold labeled by the corresponding reports. (L) The link envelope $\Psi$ resulting from Theorem~\ref{thm:eps-thick-calibrated} using $\|\cdot\|_\infty$ and $\epsilon = 1/2$, and a possible link $\psi$ which matches eq.~\eqref{eq:abstain-link} from~\cite{ramaswamy2018consistent}.  (M) An illustration of the thickened sets from Construction~\ref{const:eps-thick-link} for two sets $U \in \U$, using $\|\cdot\|_1$ and $\epsilon = 1$. (R) The $\Psi$ and $\psi$ from Theorem~\ref{thm:eps-thick-calibrated} using $\|\cdot\|_1$ and $\epsilon = 1$.}
\label{fig:abstain-links}
\end{center}
\end{figure}

\section{Additional Structure of Embeddings}
\label{sec:min-rep-sets}

We have shown in \S~\ref{sec:poly-loss-embed} a tight connection between embeddings and polyhedral losses.
Here we go beyond polyhedral losses, showing a more general necessary condition for an embedding: a surrogate embeds a discrete loss if and only if it has a polyhedral Bayes risk, or equivalently, a finite representative sets (Lemma~\ref{lem:X}).n
This result implies that the embedding condition simplifies to matching Bayes risks (Proposition~\ref{prop:embed-bayes-risks}).
It also reveals some deeper structure of embeddings, even down to the geometry of the underlying property, and the equivalence of various notions of non-redundant predictions.
In particular, we study a natural notion of a ``trimed'' loss function (Definition~\ref{def:trim-loss}), and connect this definition to both tight embeddings and non-redundancy from property elicitation (Proposition~\ref{prop:embed-iff-trims-equal}).

\subsection{Structure of polyhedral Bayes risks}

While we have focused on polyhedral losses thus far, many of our results about embeddings extend to losses with polyhedral Bayes risks, a weaker condition.
(We say a concave function is polyhedral if its negation is a polyhedral convex function.)
To see that every polyhedral loss has a polyhedral Bayes risk, recall that Theorem~\ref{thm:poly-embeds-discrete} constructs a finite representative set $\Sc$ for any polyhedral loss $L$, and thus $\risk{L} = \risk{L|_\Sc}$ by Lemma~\ref{lem:loss-restrict}, which is polyhedral.
The condition is strictly weaker: a Bayes risk may be polyhedral even if the loss itself is not.
For example, a modified hinge loss $L(r)_y = \max(r^2-1,1-ry)$
as shown in Figure~\ref{fig:modified-hinge}, which matches hinge loss on the interval $[-1,1]$ but is strictly convex outside the interval $[-2,2]$, still embeds twice 0-1 loss.

\begin{figure}
	\begin{minipage}{0.47\linewidth}
		\centering
		\includegraphics[width=0.95\linewidth]{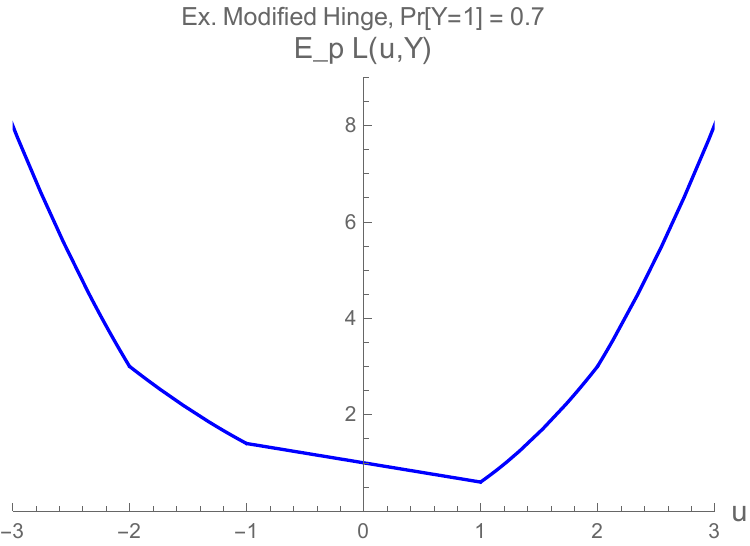}
	\end{minipage}
	\hfill
	\begin{minipage}{0.47\linewidth}
		\centering		\includegraphics[width=0.95\linewidth]{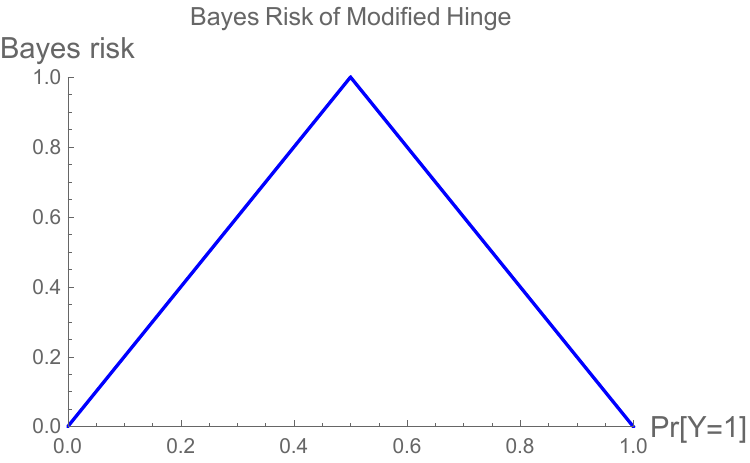}
	\end{minipage}
	\caption{(L) Expected modified hinge loss for fixed distribution; (R) Bayes risk of modified hinge still matches the Bayes risk of hinge.}
	\label{fig:modified-hinge}
\end{figure}

We now present our main structural result in Lemma~\ref{lem:X}, which will lay the foundation for the rest of this section.
The proof is in \S~\ref{subsec:lem-X-proof}.
Lemma~\ref{lem:X} observes that (minimizable) losses $L$ with polyhedral Bayes risk have finite representative sets, and derives equivalent conditions on the level sets of the property elicited by $L$ and tight embeddings.

\begin{restatable}{lemma}{lemmaX}\label{lem:X}
  Let $L: \R \to \reals^\Y_+$ be a minimizable loss with a polyhedral Bayes risk $\risk L$.
  Then $L$ has a finite representative set.
  Furthermore, letting $\Gamma = \prop{L}$, there exist finite sets
  $\V \subseteq \reals^\Y_+$ and
  $\Theta = \{\theta_v \subseteq \simplex \mid v\in\V\}$,
  both uniquely determined by $\risk{L}$ alone,
  such that
  \begin{enumerate}
  \item A set $\R'\subseteq\R$ is representative if and only if $\V \subseteq L(\R')$.\label{item:X-rep-V}
  \item A set $\R'\subseteq\R$ is minimum representative if and only if $L(\R') = \V$.\label{item:X-min-V}
  \item A set $\R'\subseteq\R$ is representative if and only if $\Theta \subseteq \{\Gamma_r \mid r \in \R'\}$.\label{item:X-rep-Theta}
  \item A set $\R'\subseteq\R$ is minimum representative if and only if $\{\Gamma_r \mid r \in \R'\} = \Theta$.\label{item:X-min-Theta}
  \item Every representative set for $L$ contains a minimum representative set for $L$.\label{item:X-rep-contain-min}
  \item The set of full-dimensional level sets of $\Gamma$ is exactly $\Theta$.\label{item:X-full-dim}
  \item For any $r \in \R$, there exists $\theta \in \Theta$ such that $\Gamma_r \subseteq \theta$.\label{item:X-redundant}
  \item $L$ tightly embeds $\ell:\R'\to\reals^\Y_+$ if and only if $\ell$ is injective and $\ell(\R') = \V$.\label{item:X-tight-embed}
  \end{enumerate}
\end{restatable}

As a finite representative set implies a polyhedral Bayes risk by Lemma~\ref{lem:loss-restrict}, Lemma~\ref{lem:X} shows that polyhedral Bayes risks are equivalent to having finite representative sets, which in turn gives an embedding by
Proposition~\ref{prop:representative-embeds-restriction}.
\begin{corollary}\label{cor:poly-risk-fin-rep}
  The following are equivalent for any minimizable loss $L:\R\to\reals^\Y_+$.
  \begin{enumerate}
  \item $\risk{L}$ is polyhedral.
  \item $L$ has a finite representative set.
  \item $L$ embeds a discrete loss.
  \end{enumerate}
\end{corollary}
From Corollary~\ref{cor:poly-risk-fin-rep}, $L$ having a finite representative set is an equivalent condition to $L$ being minimizable and $\risk{L}$ being polyhedral.
(Recall that having a finite representative set already implies minimizability.)
As it is also a more succinct condition, we will use the former in the sequel.
In particular, the implications of Lemma~\ref{lem:X} follow whenever $L$ has a finite representative set.

\subsection{Equivalent condition: matching Bayes risks}\label{subsec:match-BR}

Lemma~\ref{lem:X} leads to another appealing equivalent condition to our embedding condition in Definition~\ref{def:loss-embed}: a surrogate embeds a discrete loss if and only if their Bayes risks match.

\begin{proposition}\label{prop:embed-bayes-risks}
  Let discrete loss $\ell$ and minimizable loss $L$ be given.
  Then $L$ embeds $\ell$ if and only if $\risk{L}=\risk{\ell}$.
\end{proposition}
\begin{proof}
  Define $\Gamma = \prop{L}$ and $\gamma = \prop{\ell}$.
  
  $\implies$
  Suppose $L$ embeds $\ell$, so we have some $\Sc\subseteq \R$ which is representative for $\ell$ and an embedding $\varphi:\Sc\to\reals^d$; take $\U := \varphi(\Sc)$.
  Since $\Sc$ is representative for $\ell$, by embedding condition (ii) we have $\{\gamma_s \mid s\in\Sc\} = \{\Gamma_u \mid u\in\U\}$, so $\U$ is representative for $L$.
  By Lemma~\ref{lem:loss-restrict}, we have $\risk{\ell} = \risk{\ell|_{\Sc}}$ and $\risk{L} = \risk{L|_{\U}}$.
  As $L(\varphi(\cdot)) = \ell(\cdot)$ by embedding condition (i), for all $p\in\simplex$ we have
  \begin{equation*}
    \risk{\ell}(p) = \risk{\ell|_\Sc}(p) = \min_{r \in \Sc}\inprod{p}{\ell(r)} = \min_{r \in \Sc}\inprod{p}{L(\varphi(r))} = \min_{u \in \U}\inprod{p}{L(u)} = \risk{L|_\U}(p) = \risk{L}(p)~.
  \end{equation*}
  
  $\impliedby$
	For the reverse implication, assume $\risk{L} = \risk{\ell}$, which are polyhedral functions as $\ell$ is discrete.
  From Lemma~\ref{lem:X}(\ref{item:X-min-V}), we have some set $\V\subseteq\reals^\Y_+$ and minimum representative sets $\R^* \subseteq \R$ and $\U^* \subseteq \U$, for $\ell$ and $L$ respectively, such that $\ell(\R^*) = \V = L(\U^*)$.
  As $\R^*$ and $\U^*$ are miniumum, they cannot repeat loss vectors, and thus $|\R^*|=|\ell(\R^*)|$ and $|L(\U^*)|=|\U^*|$.
  We conclude that $\R^*$ and $\U^*$ are both in bijection with $\V$.
  The map $\varphi :\R^* \to \reals^d$, given by $\varphi(r) = u \in \U^*$ where $\ell(r) = L(u)$, is therefore well-defined.
  Condition (i) of an embedding is immediate.
  From Proposition~\ref{prop:representative-embeds-restriction}, $\ell$ embeds $\ell|_{\R^*}$ and $L$ embeds $L|_{\U^*}$, both via the identity embedding.
  Using condition (ii) from both embeddings, for all $p\in\simplex$ and $r\in\R^*$, we have
  \begin{equation*}
    r \in \gamma(p) \iff r \in \prop{\ell|_{\R^*}}(p) \iff \varphi(r) \in \prop{L|_{\U^*}}(p)
    \iff \varphi(r) \in \prop{L}(p)~,
  \end{equation*}
  giving condition (ii).
\end{proof}

Previous work from~\citet[Proposition 4]{duchi2018multiclass} realized the significance of matching Bayes risks for calibration with respect to the 0-1 loss.
Proposition~\ref{prop:embed-bayes-risks} broadens this general insight to any discrete loss.
Moreover, their result relies the Bayes risk of the surrogate being strictly concave, whereas polyhedral Bayes risks are never strictly concave.

\subsection{Trimming a loss}

Central to the structural results in Lemma~\ref{lem:X} is the existence of a canonical set of loss vectors $\V$ which match the loss vectors of any minimum representative set.
This fact may seem surprising when one considers that losses may have many mimimum representative sets.
For example, consider hinge loss with a spurious extra dimension, i.e., $L:\reals^2\to\reals^\Y$, $L((r_1, r_2))_y = \max(0,1-r_1y)$ for $\Y = \{-1,+1\}$.
Here the minimum representative sets are exactly the two-element sets of the form $\{(-1,a),(1,b)\}$ for any $a,b\in\reals$. 
Lemma~\ref{lem:X}(\ref{item:X-min-V}) states that, while the minimum representative set is not unique, its loss vectors are.

Motivated by this observation, let us define the ``trim'' of a loss to be this unique set $\V$ of loss vectors induced by any minimum representative set, which again is well-defined by Lemma~\ref{lem:X}(\ref{item:X-min-V}).
\begin{definition}[Trim]\label{def:trim-loss}
  Given a loss $L:\R \to \reals_+^\Y$ with a finite representative set, we define $\trimcover(L) = \{L(r) \mid r \in \R^*\}$ given any minimum representative set $\R^*$ for $L$.
\end{definition}

Using this notion of trimming a loss, we can again recast our embedding condition: a loss embeds another if and only if they have the same $\trimcover$.

\begin{proposition}\label{prop:embed-iff-trims-equal}
  Let $L:\reals^d\to\reals^\Y_+$ have a finite representative set, and let $\ell:\R\to\reals^\Y_+$ be a discrete loss.
  Then $L$ embeds $\ell$ if and only if $\trimcover(L) = \trimcover(\ell)$.
  Furthermore, $L$ tightly embeds $\ell$ if and only if $\ell$ is injective and $\trimcover(L) = \ell(\R)$.
\end{proposition}
\begin{proof}
  As $L$ has a finite representative set, it is minimizable.
  Proposition~\ref{prop:embed-bayes-risks} gives $L$ embeds $\ell$ if and only if $\risk L = \risk \ell$.
  If $\risk L = \risk \ell$, Lemma~\ref{lem:X}(\ref{item:X-min-V}) gives $\trim(L) = \trim(\ell)$.
  For the converse, suppose $\trim(L) = \trim(\ell) =: \V$.
  Define the discrete loss $\ell_\trim : \V \to \V, v\mapsto v$.
  Then $\ell_\trim$ is injective and $\ell_\trim(\V) = \V$, so from Lemma~\ref{lem:X}(\ref{item:X-tight-embed}), both $L$ and $\ell$ tightly embed $\ell_\trim$.
  We conclude $\risk L = \risk{\ell_\trim} = \risk \ell$ from Proposition~\ref{prop:embed-bayes-risks}.
  The second statement also follows directly from Lemma~\ref{lem:X}(\ref{item:X-tight-embed}).
\end{proof}

\subsection{Minimum representative sets and non-redundancy}

The condition that a representative set be minimum implies that one has identified exactly the ``active'' reports of a loss, in some sense.
We now relate this condition to another natural notion from the property elicitation literature: non-redundancy~\cite{frongillo2014general,lambert2018elicitation}.
Intuitively, a loss is non-redundant if no report is weakly dominated by another report.

\begin{definition}[Non-redundancy]\label{def:nonredundant}
  A loss $L : \R \to \reals^\Y_+$ eliciting $\Gamma:\simplex \toto \R$ is \emph{redundant} if there are reports $r, r' \in \R$ with $r \neq r'$ such that $\Gamma_r \subseteq \Gamma_{r'}$, and \emph{non-redundant} otherwise.
\end{definition}

From the structural result of Lemma~\ref{lem:X}, we can see that in fact these two notions are equivalent when $L$ has a polyhedral Bayes risk.
\begin{proposition}\label{prop:tfae-min-rep-nonredundant}
  Let $L:\R\to\reals^\Y_+$ have a finite representative set $\R'$.
  Then $\R'$ is a minimum representative set for $L$ if and only if $L|_{\R'}$ is non-redundant.
\end{proposition}
\begin{proof}
  Let $\Gamma = \prop{L}$.
  Suppose first that $L|_{\R'}$ is redundant.
  Then there exist $r,r' \in \R'$ such that $\Gamma_r \subseteq \Gamma_{r'}$.
  Thus, for all $p \in \Gamma_r$, we have $\{r, r'\} \subseteq \Gamma(p)$.
  Therefore $\R' \setminus \{r\}$ still a representative set, so $\R'$ is not minimum.

  Now suppose $L|_{\R'}$ is non-redundant.
  As $\R'$ is a representative set, Lemma~\ref{lem:X}(\ref{item:X-rep-contain-min}) gives some minimum representative set $\Sc \subseteq \R'$.
  Suppose we had some $r \in \R' \setminus \Sc$.
  Now Lemma~\ref{lem:X}(\ref{item:X-min-Theta},\ref{item:X-redundant}) gives some $s\in\Sc$ such that $\Gamma_r \subseteq \Gamma_s$, which contradicts $L|_{\R'}$ being non-redundant.
  We conclude $L(\Sc)=L(\R')$, meaning $\R'$ is a minimum representative set.
\end{proof}

\begin{corollary}\label{cor:tight-embed-min-rep}
  Let loss $L:\R\to\reals^\Y_+$ with finite representative set $\R'$ be given.
  Then $L$ tightly embeds $L|_{\R'}$ if and only if $L|_{\R'}$ is non-redundant.
\end{corollary}

In fact, we can show something stronger: the reports in minimum representative sets are precisely those which are not strictly redundant.
To formalize this statement, given $\Gamma : \simplex \toto \R$, let $\red(\Gamma) := \{r\in\R \mid \exists r'\in\R,\; \Gamma_r \subsetneq \Gamma_{r'}\}$ be the set of strictly redundant reports.
Similarly, for minimizable $L$, let $\red(L) := \red(\prop L)$.

\begin{proposition}
  Let $L : \R \to \reals^\Y_+$ have a finite representative set.
  Let $\R'$ be the union of all minimum representative sets for $L$.
  Then $\R' = \R \setminus \red(L)$.
\end{proposition}

\begin{proof}
  Let $\Gamma = \prop L$.
  Let $\Sc$ be a minimum representative set for $L$, and let $s\in\Sc$.
  Suppose for a contradiction that $s\in\red(\Gamma)$.
  Then we have some $r\in\R$ with $\Gamma_s \subsetneq \Gamma_r$.
  From Lemma~\ref{lem:X}(\ref{item:X-min-Theta},\ref{item:X-redundant}) we have some $s'\in\Sc$ such that $\Gamma_r \subseteq \Gamma_{s'}$.
  But now $\Gamma_s \subsetneq \Gamma_r \subseteq \Gamma_{s'}$, contradicting $\Sc$ being minimum representative.
  Thus $\Sc \subseteq \R \setminus \red(\Gamma)$.

  For the reverse inclusion, let $r\in\R\setminus\red(\Gamma)$.
  Let $\Sc$ again be a minimum representative set for $L$.
  From Lemma~\ref{lem:X}(\ref{item:X-min-Theta},\ref{item:X-redundant}), we have some $s\in\Sc$ such that $\Gamma_r \subseteq \Gamma_s$.
  By definition of $\red(L)$, we conclude $\Gamma_r = \Gamma_s$.
  Now take $\Sc' = (\Sc \setminus \{s\}) \cup \{r\}$, that is, the same set of reports with $r$ replacing $s$.
  We have $\{\Gamma_s \mid s\in\Sc\} = \{\Gamma_{s'} \mid s'\in\Sc'\}$, and thus $\Sc'$ is a minimum representative for $L$ by Lemma~\ref{lem:X}(\ref{item:X-min-Theta}).
  As $r\in\Sc'$, we have $r \in \R'$ and we are done.
\end{proof}

As a corollary, we can state another characterization of $\trim$ in terms of redundant reports.
The result follows immediately from the definition of $\trim$.

\begin{corollary}\label{cor:trim-loss-red}
  Let $L : \R \to \reals^\Y_+$ have a finite representative set.
  Then $\trimcover(L) = L(\R \setminus \red(L))$.
\end{corollary}

This result motivates the analogous definition for properties, $\trimred(\Gamma) := \{\Gamma_r \mid r \in \R\setminus\red(\Gamma)\}$.
We leverage this definition next, to study embeddings at the property level.

\subsection{A property elicitation perspective on trimmed losses}

We conclude this section with a similar structural result about the properties embedded by another property.
We say a property $\Gamma:\simplex\toto\reals^d$ embeds a finite property $\gamma:\simplex\toto\R$ if condition (ii) of Definition~\ref{def:loss-embed} holds.
In other words, $\Gamma$ embeds $\gamma$ if we have some representative set $\Sc\subseteq\R$ for $\gamma$ and embedding $\varphi:\Sc\to\reals^d$ such that for all $s\in\Sc$ we have $\gamma_s = \Gamma_{\varphi(s)}$.

Roughly, our result is as follows.
First, if $\Gamma$ embeds $\gamma$ and $\gamma$ is non-redundant, the level sets of $\Gamma$ must all be redundant relative to $\gamma$.
In other words, $\Gamma$ is exactly the property $\gamma$ up to relabelling reports, just with other reports filling in the gaps between the embedded reports of $\gamma$.
When working with convex losses, these extra reports often arise in the convex hull of the embedded reports.
In this sense, we can regard embedding as only a slight departure from direct elicitation: if a loss $L$ elicits $\Gamma$ which embeds $\gamma$, we can almost think of $L$ as eliciting $\gamma$ itself.
Finally, we have an important converse: if $\Gamma$ has finitely many full-dimensional level sets, or equivalently, if $\trimred(\Gamma)$ is finite, then $\Gamma$ must embed some finite elicitable property with those same level sets.
The statements about level sets make use of another corollary of Proposition~\ref{prop:embed-iff-trims-equal}, stated for properties.
\begin{corollary}\label{cor:trim-prop-red}
  Let $\Gamma : \simplex \toto \R$ be an elicitable property with a finite representative set.
  Then $\trimred(\Gamma)$ is the set of full-dimensional level sets of $\Gamma$.
\end{corollary}
\begin{proof}
  Let $L$ elicit $\Gamma$.
  From Lemma~\ref{lem:X}(\ref{item:X-min-Theta},\ref{item:X-full-dim}), for any finite minumum representative set $\Sc\subseteq\R$, the set $\{\Gamma_s\mid s\in\Sc\}$ is exactly the set of full-dimensional level sets $\Theta$ of $\Gamma$.
  From Proposition~\ref{prop:tfae-min-rep-nonredundant}, we have $r \in \R\setminus \red(\Gamma)$ if and only if $r$ is an element of some minimum representative set.
  As $\Gamma$ has at least one minimum representative set, we conclude $\trimred(\Gamma) = \{\Gamma_r \mid r\in \R\setminus\red(\Gamma)\} = \Theta$.  
\end{proof}

\begin{proposition}\label{prop:embed-trim}
  Let $\Gamma:\simplex\toto\reals^d$ be an elicitable property.
  The following are equivalent:
  \begin{enumerate}\setlength{\itemsep}{0pt}
  \item $\Gamma$ embeds a elicitable finite property $\gamma:\simplex \toto \R$.
  \item $\trimred(\Gamma)$ is a finite set.
  \item There is a finite minimum representative set $\U$ for $\Gamma$.
  \item There is a finite set of full-dimensional level sets $\hat\Theta$ of $\Gamma$, and $\cup\,\hat\Theta = \simplex$.
  \end{enumerate}
  Moreover, when any of the above hold, $\trimred(\gamma) = \trimred(\Gamma) = \{\Gamma_u \mid u\in\U\} = \hat\Theta$.
\end{proposition}

\begin{proof}
  Let $L$ be a fixed loss eliciting $\Gamma$, so that in particular $\risk L$ is fixed.
  By definition of elicitable properties, $L$ is minimizable.
  In each case, we will show that $\risk L$ is polyhedral (or equivalently, that $L$ has a finite representative set), and thus Lemma~\ref{lem:X} will give us the set $\Theta$ of full-dimensional level sets of $\Gamma$, uniquely determined by $\risk L$.
  We will prove $1 \Rightarrow 2 \Rightarrow 3 \Rightarrow 4 \Rightarrow 1$, and in each case show that the relevant set of level sets is equal to $\Theta$, giving the result.

  $1 \Rightarrow 2$:
  Let $\Sc$ be the representative set for $\gamma$ and $\varphi:\Sc\to\reals^d$ the embedding.
  Since $\Sc$ is finite, $\varphi(\Sc)$ is a finite representative set for $\Gamma$ (and $L$; thus, $\risk L$ is polyhedral).
  Corollary~\ref{cor:trim-prop-red} now gives $\trimred(\Gamma) = \Theta$, which is finite, showing Case 2.

  $2 \Rightarrow 3$:
  If $\trimred(\Gamma)$ is finite, then in particular we have a finite set of reports $\Sc \subseteq \reals^d$ such that $\trimred(\Gamma) = \{\Gamma_s \mid s\in\Sc\}$.
  As $\Gamma$ is elicitable, $\reals^d$ is representative for $\Gamma$.
  By definition of $\trimred$, we have $\simplex = \cup_{r\in\reals^d} \Gamma_r = \cup \trimred(\Gamma) = \cup_{s\in\Sc} \Gamma_s$, and therefore $\Sc$ is representative for $\Gamma$ and for $L$.
  As $\Sc$ is finite, we have $\risk L$ polyhedral.
  From Lemma~\ref{lem:X}(\ref{item:X-rep-contain-min}), we have some minimum representative set $\U\subseteq\Sc$ for $L$ and $\Gamma$, implying statement 3.
  Moreover, Lemma~\ref{lem:X}(\ref{item:X-min-Theta},\ref{item:X-full-dim}) gives $\{\Gamma_u \mid u\in\U\} = \Theta$.

  $3 \Rightarrow 4$:
  Let $\U$ be a finite minimum representative set for $\Gamma$.
  Then $\risk L = \risk{L|_\U}$ is polyhedral.
  Lemma~\ref{lem:X}(\ref{item:X-min-Theta},\ref{item:X-full-dim}) once again gives $\{\Gamma_u \mid u\in\U\} = \Theta$.
  We simply let $\hat\Theta = \Theta$, giving statement 4 as $\U$ is representative.

  $4 \Rightarrow 1$:
  Let $\Sc\subseteq\R$ such that $\{\Gamma_s \mid s\in\Sc\} = \hat \Theta$.
  Then $\Sc$ is representative for $\Gamma$ and $L$, as $\cup\hat\Theta = \simplex$.
  Again, this yields a finite representative set for $L$.
  Lemma~\ref{lem:loss-restrict} now states that $L$ embeds $L|_\Sc$, so $\Gamma$ embeds $\gamma := \Gamma|_\Sc$, giving Case 1.
  Finally, Corollary~\ref{cor:trim-prop-red} gives $\trimred(\gamma) = \Theta$.
\end{proof}

As a final observation, recall that a property $\Gamma$ elicited by a polyhedral loss has a finite range, in the sense that there are only finitely many optimal sets $\Gamma(p)$ for $p\in\simplex$ (Lemma~\ref{lem:polyhedral-range-gamma}).
Proposition~\ref{prop:embed-trim} shows the dual statement: there are only finitely many level sets $\Gamma_u$ for $u\in\reals^d$.
In other words, both $\Gamma$ and $\Gamma^{-1}$ have a finite range as multivalued maps.

\section{Polyhedral Indirect Elicitation Implies Consistency}
\label{sec:poly-ie-consistency}

Our last result concerns indirect elicitation as a necessary condition for consistency when restricting to polyhedral losses.
Intuitively, a loss $L$ indirectly elicits a property $\gamma$ if we can compute $\gamma$ from $\prop L$.
To formalize the condition, we use the notion of a property refining another from~\citet{frongillo2014general}.

\begin{definition}[Refines]
	Let $\Gamma:\simplex \toto \R$ and $\Gamma':\simplex\toto \R'$.
	Then $\Gamma$ \emph{refines} $\Gamma'$ if for all $r \in \R$, there exists $r' \in \R'$ such that $\Gamma_{r} \subseteq \Gamma'_{r'}$.
\end{definition}
Equivalently, $\Gamma$ refines $\Gamma'$ if there is some ``link'' function $\psi:\R\to\R'$ such that $r\in\Gamma(p) \implies \psi(r) \in \Gamma'(p)$ for all $p\in\simplex$.
We will use the fact that refinement is transitive: if $\Gamma$ refines $\Gamma'$ and $\Gamma'$ refines $\Gamma''$, then $\Gamma$ refines $\Gamma''$.

\begin{definition}[Indirectly elicits]
  A loss $L$ \emph{indirectly elicits} a property $\gamma$ if $\prop L$ refines $\gamma$.
\end{definition}

It is straightforward to verify that consistency, and therefore calibration, implies indirect elicitation~\citep{finocchiaro2021unifying,agarwal2015consistent,steinwart2008support}.
Indirect elicitation may appear much weaker than calibration, since in particular it does not depend on the loss except through the property it elicits, and thus only depends on the exact minimizers of the loss.
Surprisingly, for minimizable polyhedral surrogates, we show the converse: indirect elicitation implies calibration, and therefore consistency.

A useful lemma is that for minimizable polyhedral losses, indirect elicitation must always pass through an embedding.
This result holds more generally whenever $L$ has a finite representative set, as in \S~\ref{sec:min-rep-sets}.
\begin{lemma}\label{lem:ie-iff-embeds-refinement}
  Let $L$ be a minimizable polyhedral loss.
  Then $L$ indirectly elicits a property $\gamma$ if and only if $L$ tightly embeds a discrete loss $\ell$ such that $\prop \ell$ refines $\gamma$.
\end{lemma}
\begin{proof}
  Let $L:\reals^d\to\reals^\Y_+$ be polyhedral, and $\Gamma = \prop L$.
  Then $L$ tightly embeds a discrete loss from Lemma~\ref{lem:X}(\ref{item:X-tight-embed}).
  Furthermore, Lemma~\ref{lem:X}(\ref{item:X-min-Theta},\ref{item:X-redundant},\ref{item:X-tight-embed}) implies that $\prop L$ refines $\prop \ell$ for any discrete loss $\ell$ that $L$ tightly embeds.

  We claim that, for any property $\gamma$, and any loss $\ell$ that $L$ tightly embeds, $\prop L$ refines $\gamma$ if and only if $\prop \ell$ refines $\gamma$.
  If $\prop \ell$ refines $\gamma$, then $\prop L$ refines $\gamma$ by transitivity.
  For the other direction, Lemma~\ref{lem:X}(\ref{item:X-min-Theta},\ref{item:X-tight-embed}) shows that the level sets of $\prop \ell$ are contained in the set $\{\Gamma_u \mid u\in\reals^d\}$.
  Thus, if $\prop L$ refines $\gamma$, then in particular $\prop \ell$ refines $\gamma$.
  The result now follows immediately from the claim.
\end{proof}

\begin{theorem}\label{thm:poly-ie-implies-consistent}
	Let $L$ be a minimizable polyhedral loss which indirectly elicits a finite property $\gamma$.
  For any loss $\ell$ eliciting $\gamma$, there exists a link $\psi$ such that $(L, \psi)$ is calibrated (and consistent) with respect to $\ell$.
\end{theorem}
\begin{proof}
	Let $L:\reals^d \to \reals^\Y_+$ be a polyhedral loss indirectly eliciting $\gamma: \simplex \toto \R$, and let $\ell$ be a discrete loss eliciting $\gamma$.
  By Lemma~\ref{lem:ie-iff-embeds-refinement}, $L$ tightly embeds a discrete loss $\ell^\emb:\R^\emb\to\reals^\Y_+$ such that $\gamma^\emb := \prop{\ell^\emb}$ refines $\gamma$.
  From refinement, we can define a function $\psi^\R: \R^\emb \to \R$ such that for all $r\in\R^\emb$ and $p\in\simplex$ we have $r\in\gamma^\emb(p) \implies \psi^\R(r) \in \gamma(p)$. 
  Finally, Theorem~\ref{thm:link-main} gives a link function $\psi^\emb : \reals^d \to \R^\emb$ such that $(L,\psi^\emb)$ is calibrated with respect to $\ell^\emb$.

  Consider $\psi := \psi^\R \circ \psi^\emb$ and fix $p\in\simplex$.
	For any $u\in\reals^d$, if $\psi^\emb(u) \in \gamma^\emb(p)$, then $\psi(u) = \psi^\R(\psi^\emb(u)) \in \gamma(p)$ by definition of $\psi$ and $\psi^\R$.
  Contrapositively,
  $\psi(u) \notin \gamma(p) \implies \psi^\emb(u) \notin \gamma^\emb(p)$.
  Thus, we have
  \begin{equation}
    \label{eq:link-set-inclusion}
    \{u\in\reals^d \mid \psi(u) \not \in \gamma(p) \} \subseteq \{u\in\reals^d \mid \psi^\emb(u) \not \in \gamma^\emb(p) \}~.
  \end{equation}
  Combining eq.~\eqref{eq:link-set-inclusion} with the fact that $(L,\psi^\emb)$ is calibrated with respect to $\ell^\emb$, we have
	\begin{align*}
\inf_{u\in\reals^d : \psi(u) \not \in \gamma(p)} \inprod{p}{L(u)} \geq	\inf_{u\in\reals^d : \psi^\emb(u) \not \in \gamma^\emb(p)} \inprod{p}{L(u)} > \inf_{u\in\reals^d}\inprod{p}{L(u)}~,
	\end{align*}
  showing calibration of $\psi$.
	Consistency follows as calibration and consistency are equivalent in this setting~\citep{ramaswamy2016convex}.
\end{proof}

Theorem~\ref{thm:poly-ie-implies-consistent} gives a somewhat surprising result: despite the fact that indirect elicitation appears to be a somewhat weak necessary condition for consistency in general, the two conditions are equivalent for polyhedral surrogates.

\section{Conclusions} \label{sec:conclusion}

Several directions for future work remain.
We show in Theorem~\ref{thm:poly-ie-implies-consistent} that indirect elicitation is equivalent to consistency when restricting to the class of polyhedral surrogates; we would like to identify other classes of surrogates for which this equivalence holds.
It would also be interesting to explore embeddings through the lens of superprediction sets~\citep{williamson2014geometry}.
Finally, it is important for applications to understand the minimum prediction dimension $d$ of a consistent convex surrogate $L:\reals^d\to \reals^\Y_+$ for a given target problem, also called its elicitation complexity.
One approach to this question is to first understand the minimum $d$ for which an embedding $L$ exists, a study initiated by \citet{finocchiaro2020embedding}, and then relate this dimension to polyhedral, or general convex, elicitation complexity.

\subsection*{Acknowledgements}
We thank Arpit Agarwal and Peter Bartlett for many early discussions and insights,
Stephen Becker for a reference to Hoffman constants,
and Nishant Mehta, Enrique Nueve, and Anish Thilagar for other suggestions.
This material is based upon work supported by the National Science Foundation under Grants No.\ 1657598 and No.\ DGE 1650115.
\newpage
\bibliographystyle{plainnat}
\bibliography{diss,extra}

\appendix

\newpage
\section{Power diagrams}\label{app:power-diagrams}
First, we present several definitions from Aurenhammer~\cite{aurenhammer1987power}.
\begin{definition}\label{def:cell-complex}
  A \emph{cell complex} in $\reals^d$ is a set $C$ of faces (of dimension $0,\ldots,d$) which (i) union to $\reals^d$, (ii) have pairwise disjoint relative interiors, and (iii) any nonempty intersection of faces $F,F'$ in $C$ is a face of $F$ and $F'$ and an element of $C$.
\end{definition}

\begin{definition}\label{def:power-diagram}
  Given sites $s_1,\ldots,s_k\in\reals^d$ and weights $w_1,\ldots,w_k \geq 0$, the corresponding \emph{power diagram} is the cell complex given by
  \begin{equation}
    \label{eq:pd}
    \cell(s_i) = \{ x \in\reals^d : \forall j \in \{1,\ldots,k\} \, \|x - s_i\|^2 - w_i \leq \|x - s_j\|^2 - w_j\}~.
  \end{equation}
\end{definition}

\begin{definition}\label{def:affine-equiv}
  A cell complex $C$ in $\reals^d$ is \emph{affinely equivalent} to a (convex) polyhedron $P \subseteq \reals^{d+1}$ if $C$ is a (linear) projection of the faces of $P$.
\end{definition}

Proposition~\ref{prop:embed-bayes-risks}, focuses on matching the values of Bayes Risks, while the following result from~\citet{aurenhammer1987power} allows us to move towards understanding the projection of the Bayes Risk onto the simplex $\simplex$.
In particular, one can consider the epigraph of a polyhedral convex function on $\reals^d$ and the projection down to $\reals^d$; in this case we call the resulting power diagram \emph{induced} by the convex function.

\begin{theorem}[Aurenhammer~\cite{aurenhammer1987power}]\label{thm:aurenhammer}
	A cell complex is affinely equivalent to a convex polyhedron if and only if it is a power diagram.
\end{theorem}

We extend Theorem~\ref{thm:aurenhammer} to a weighted sum of convex functions, showing that the induced power diagram is the same for any choice of strictly positive weights.

\begin{lemma}\label{lem:polyhedral-pd-same}
	Let $f_1,\ldots,f_m:\reals^d\to\reals$ be polyhedral convex functions.
	The power diagram induced by $\sum_{i=1}^m p_i f_i$ is the same for all $p \in \inter(\simplex)$.
\end{lemma}
\begin{proof}
	For any polyhedral convex function $g$ with epigraph $P$, the proof of~\citet[Theorem 4]{aurenhammer1987power} shows that the power diagram induced by $g$ is determined by the facets of $P$.
	Let $F$ be a facet of $P$, and $F'$ its projection down to $\reals^d$.
	It follows that $g|_{F'}$ is affine, and thus $g$ is differentiable on $\inter(F')$ with constant derivative $d\in\reals^d$.
	Conversely, for any subgradient $d'$ of $g$, the set of points $\{x\in\reals^d : d'\in\partial g(x)\}$ is the projection of a face of $P$; we conclude that $F = \{(x,g(x))\in\reals^{d+1} : d\in\partial g(x)\}$ and $F' = \{x\in\reals^d : d\in\partial g(x)\}$.
	
	Now let $f := \sum_{i=1}^k f_i$ with epigraph $P$, and $f' := \sum_{i=1}^k p_i f_i$ with epigraph $P'$.
	By Rockafellar~\cite{rockafellar1997convex}, $f,f'$ are polyhedral.
	We now show that $f$ is differentiable whenever $f'$ is differentiable:
	\begin{align*}
	\partial f(x) = \{d\}
	&\iff \sum_{i=1}^k \partial f_i(x) = \{d\} \\
	&\iff \forall i\in\{1,\ldots,k\}, \; \partial f_i(x) = \{d_i\} \\
	&\iff \forall i\in\{1,\ldots,k\}, \; \partial p_i f_i(x) = \{p_id_i\} \\
	&\iff \sum_{i=1}^k \partial p_if_i(x) = \left\{\sum_{i=1}^k p_id_i\right\} \\
	&\iff \partial f'(x) = \left\{\sum_{i=1}^k p_id_i\right\}~.
	\end{align*}
	From the above observations, every facet of $P$ is determined by the derivative of $f$ at any point in the interior of its projection, and vice versa.
	Letting $x$ be such a point in the interior, we now see that the facet of $P'$ containing $(x,f'(x))$ has the same projection, namely $\{x'\in\reals^d : \nabla f(x) \in \partial f(x')\} = \{x'\in\reals^d : \nabla f'(x) \in \partial f'(x')\}$.
	Thus, the power diagrams induced by $f$ and $f'$ are the same.
	The conclusion follows from the observation that the above held for any strictly positive weights $p$, and $f$ was fixed.
\end{proof}

We now include the full proof of Lemma~\ref{lem:polyhedral-range-gamma}.

\polyhedralrangegamma*
\begin{proof}
	First, observe that $L: \reals^d \to \reals^\Y_+$ is finite and bounded from below (by $\vec 0$), and thus its infimum is finite. 
	Therefore, we can apply \citet[Corollary 19.3.1]{rockafellar1997convex} to conclude that its infimum is attained for all $p \in \simplex$ and is therefore minimizable; thus, elicits a property.
	
	For all $p$, let $P(p)$ be the epigraph of the convex function $u\mapsto \inprod{p}{L(u)}$.
	From Lemma~\ref{lem:polyhedral-pd-same}, we have that the power diagram $D_\Y$ induced by the projection of $P(p)$ onto $\reals^d$ is the same for any $p\in\inter(\simplex)$.
	Let $\F_\Y$ be the set of faces of $D_\Y$, which by the above are the set of faces of $P(p)$ projected onto $\reals^d$ for any $p\in\inter(\simplex)$.
	
	We claim for all $p\in\inter(\simplex)$, that $\Gamma(p) \in \F_\Y$.
	To see this, let $u \in \Gamma(p)$, and $u' = (u,\inprod{p}{L(u)}) \in P(p)$.
	The optimality of $u$ is equivalent to $u'$ being contained in the face $F$ of $P(p)$ exposed by the normal $(0,\ldots,0,-1)\in\reals^{d+1}$.
	Thus, $\Gamma(p) = \argmin_{u\in\reals^d} \inprod{p}{L(u)}$ is a projection of $F$ onto $\reals^d$, which is an element of $\F_\Y$.
	
	Now for $p \not \in \inter(\simplex)$, consider $\Y'\subsetneq \Y$, $\Y'\neq\emptyset$.
	Applying the above argument, we have a similar guarantee: a finite set $\F_{\Y'}$ such that $\Gamma(p) \in \F_{\Y'}$ for all $p$ with support exactly $\Y'$.
	Taking $\F = \bigcup\{\F_{\Y'} | \Y'\subseteq\Y, \Y'\neq\emptyset\}$, we have for all $p\in\simplex$ that $\Gamma(p) \in \F$, giving $\U \subseteq \F$.
	As $\F$ is finite, so is $\U$, and the elements of $\U$ are closed polyhedra as faces of $D_{\Y'}$ for some $\Y'\subseteq\Y$.
\end{proof}

\section{Equivalence of Separation and Calibration for Polyhedral Surrogates}
\label{sec:equiv-sep-calib}

We recall that Theorem \ref{thm:link-main} states that, if a polyhedral $L$ embeds a discrete $\ell$, then there exists a calibrated link $\psi$.
Theorem \ref{thm:link-main} is directly implied by the combination of Theorem \ref{thm:calibrated-separated}, that calibration is equivalent to separation (Definition \ref{def:sep-link}); and Theorem \ref{thm:thickened-separated}, existence of a separated link.
Theorem \ref{thm:calibrated-separated} is proven in this section and Theorem \ref{thm:thickened-separated} is proven in Appendix \ref{app:sep-link-exists}.

Throughout we will work with the two \emph{regret} functions:
the \emph{surrogate regret} $R_L(u,p) = \inprod{p}{L(u)} - \risk{L}(p)$, and similarly the \emph{target regret} $R_{\ell}(r,p) = \inprod{p}{\ell(r)} - \risk{\ell}(p)$.
In fact, the results in this section can be extended to surrogate regret bounds; see \citet{frongillo2021surrogate}.

We first show one direction: any calibrated link from a polyhedral surrogate to a discrete target must be $\epsilon$-separated.
The proof follows a similar argument to that of~\citet[Lemma 6]{tewari2007consistency}.
\begin{lemma}\label{lemma:calibrated-eps-sep}
  Let polyhedral surrogate $L:\reals^d \to \reals^\Y_+$, discrete loss $\ell:\R\to\reals^\Y_+$, and link $\psi:\reals^d\to\R$ be given such that $(L,\psi)$ is calibrated with respect to $\ell$.
  Then there exists $\epsilon>0$ such that $\psi$ is $\epsilon$-separated with respect to   $\prop{L}$ and $\prop{\ell}$.
\end{lemma}
\begin{proof}
  Let $\Gamma := \prop{L}$ and $\gamma := \prop{\ell}$.
  Suppose that $\psi$ is not $\epsilon$-separated for any $\epsilon>0$.
  Then letting $\epsilon_i := 1/i$ we have sequences $\{p_i\}_i \subset \simplex$ and  $\{u_i\}_i \subset \reals^d$ such that for all $i\in\mathbb N$ we have both $\psi(u_i) \notin \gamma(p_i)$ and $d_\infty(u_i,\Gamma(p_i)) \leq \epsilon_i$.
  First, observe that there are only finitely many values for $\gamma(p_i)$ and $\Gamma(p_i)$, as $\R$ is finite and $L$ is polyhedral (from Lemma~\ref{lem:polyhedral-range-gamma}).
  Thus, there must be some $p\in\simplex$ and some infinite subsequence indexed by $j\in J \subseteq \mathbb N$ where
  for all $j\in J$, we have $\psi(u_j) \notin \gamma(p)$ and $\Gamma(p_j) = \Gamma(p)$.

  Next, observe that, as $L$ is polyhedral, the expected loss $\inprod{p}{L(u)}$ is $\beta$-Lipschitz in $\|\cdot\|_\infty$ for some $\beta>0$.
  Thus, for all $j\in J$, we have
  \begin{align*}
    d_\infty(u_i,\Gamma(p)) \leq \epsilon_j
    &\implies \exists u^*\in\Gamma(p) \|u_j-u^*\|_\infty \leq \epsilon_j
    \\
    &\implies \left| \inprod{p}{L(u_j)} - \inprod{p}{L(u^*)} \right| \leq \beta\epsilon_j
    \\
    &\implies \left| \inprod{p}{L(u_j)} - \risk{L}(p) \right| \leq \beta\epsilon_j~.
  \end{align*}
  Finally, for this $p$, we have
  \begin{align*}
    \inf_{u:\psi(u)\notin\gamma(p)} \inprod{p}{L(u)}
    \leq
    \inf_{j\in J} \inprod{p}{L(u_j)}
    =
    \risk{L}(p)~,
  \end{align*}
  contradicting the calibration of $\psi$.
\end{proof}

For the other direction, we will make use of Hoffman constants for systems of linear inequalities.
See \citet{zalinescu2003sharp} for a modern treatment.
\begin{theorem}[Hoffman constant \cite{hoffman1952approximate}]
  \label{thm:hoffman}
  Given a matrix $A\in\reals^{m\times n}$, there exists some smallest $H(A)\geq 0$, called the \emph{Hoffman constant} (with respect to $\|\cdot\|_\infty$), such that for all $b\in\reals^m$ and all $x\in\reals^n$,
  \begin{equation}
    \label{eq:hoffman}
    d_\infty(x,S(A,b)) \leq H(A) \|(A x - b)_+\|_\infty~,
  \end{equation}
  where $S(A,b) = \{x\in\reals^n \mid A x \leq b\}$ and $(u)_+ := \max(u,0)$ component-wise.
\end{theorem}

\begin{lemma}\label{lemma:hoffman-polyhedral}
  Let $L: \reals^d \to \reals_+^{\Y}$ be a polyhedral loss with $\Gamma = \prop{L}$.
  Then for any fixed $p$, there exists some smallest constant $H_{L,p} \geq 0$ such that $d_{\infty}(u,\Gamma(p)) \leq H_{L,p} R_L(u,p)$ for all $u \in \reals^d$.
\end{lemma}
\begin{proof}
  Since $L$ is polyhedral, there exist $a_1,\ldots,a_m \in \reals^d$ and $c\in\reals^m$ such that we may write $\inprod{p}{L(u)} = \max_{1\leq j\leq m} a_j \cdot u + c_j$.
  Let $A \in \reals^{m\times d}$ be the matrix with rows $a_j$, and let $b = \risk{L}(p)\ones - c$, where $\ones\in\reals^m$ is the all-ones vector.
  Then we have
  \begin{align*}
    S(A,b)
    &:= \{u\in\reals^d \mid A u \leq b\}
    \\
    &= \{u\in\reals^d \mid A u + c \leq \risk{L}(p)\ones\}
    \\
    &= \{u\in\reals^d \mid \forall i\, (A u + c)_i \leq \risk{L}(p)\}
    \\
    &= \{u\in\reals^d \mid \max_i \;(A u + c)_i \leq \risk{L}(p)\}
    \\
    &= \{u\in\reals^d \mid \inprod{p}{L(u)} \leq \risk{L}(p)\}
    \\
    & = \Gamma(p)~.
  \end{align*}
  Similarly, we have $\max_i\; (A u - b)_i = \inprod{p}{L(u)} - \risk{L}(p) = \regret{L}{u}{p} \geq 0$.
  Thus,
  \begin{align*}
    \|(Au - b)_+\|_\infty
    &= \max_i\; ((Au - b)_+)_i
    \\
    &= \max((Au - b)_1,\ldots,(Au - b)_m, 0)
    \\
    &= \max(\max_i\; (Au - b)_i, \, 0)
    \\
    &= \max_i\; (Au - b)_i
    \\
    &= \regret{L}{u}{p}~.
  \end{align*}
  Now applying Theorem~\ref{thm:hoffman}, we have
  \begin{align*}
    d_\infty(u,\Gamma(p))
    &=    d_\infty(u,S(A,b))
    \\
    &\leq H(A) \|(Au-b)_+\|_\infty
    \\
    &= H(A) \regret{L}{u}{p}~.\qedhere
  \end{align*}
\end{proof}

We are now ready to prove Theorem \ref{thm:calibrated-separated} as desired.
\calibratedseparated*
\begin{proof}
  Let $\gamma=\prop{\ell}$ and $\Gamma=\prop{L}$.
  From Lemma~\ref{lemma:calibrated-eps-sep}, calibration implies $\epsilon$-separation.
  For the converse, suppose $\psi$ is $\epsilon$-separated with respect to $L$ and $\ell$.
  Fix $p\in\simplex$.
  To show calibration, it suffices to find a positive lower bound for $R_L(u,p)$ that holds for all $u\in\reals^d$ with $\psi(u) \notin \gamma(p)$.
  
  Applying the definition of $\epsilon$-separated and Lemma~\ref{lemma:hoffman-polyhedral}, $\psi(u) \notin \gamma(p)$ implies
  \begin{align*}
    \epsilon &<    d_{\infty}(u,\Gamma(p)) \leq H_{L,p} R_L(u,p) .
  \end{align*}
  Let $C_{\ell} = \max_{r,p} R_{\ell}(r,p)$.
  Then $R_{\ell}(\psi(u),p) \leq C_{\ell} \leq \frac{C_{\ell} H_{L,p}}{\epsilon} R_L(u,p)$.

  If $H_{L,p} = 0$, then for all $u\in\reals^d$ we have $R_\ell(\psi(u),p) = 0$, so calibration for this $p$ is trivial.
  Similarly, if $C_\ell = 0$, then $R_\ell(r,p) = 0$ for all $r\in\R$, so again $R_\ell(\psi(u),p) = 0$ for all $u\in\reals^d$.

  Now assume $C_\ell > 0$ and $H_{L,p} > 0$.
  Let $C'_{\ell,p} \doteq \min_{r \notin \gamma(p)} R_\ell(r,p) > 0$.
  (As we assume $C_\ell > 0$, we must have $\gamma(p) \neq \R$, so the minimum is attained.)
  Then for all $u$ such that $\psi(u) \notin \gamma(p)$, we have $R_\ell(\psi(u),p) \geq C'_{\ell,p}$.
  Rearranging, we have
  \[ \psi(u) \notin \gamma(p) \implies R_L(u,p) \geq \frac{C'_{\ell,p} \epsilon}{C_\ell H_{L,p}} > 0~.\]
  Thus, $\inf_{u : \psi(u) \notin \gamma(p)} \inprod{L(u)}{p} > \risk{L}(p)$.
  Since the above holds for all $p\in\simplex$, $\psi$ is calibrated.
\end{proof}

\section{Existence of a Separated Link} \label{app:sep-link-exists}
In this section, we prove Theorem \ref{thm:thickened-separated}, as discussed at the beginning of Appendix \ref{sec:equiv-sep-calib}.

We define some notation and assumptions to be used throughout this section.
Let some norm $\|\cdot\|$ on finite-dimensional Euclidean space be given.
Given a set $T$ and a point $u$, let $d(T,u) = \inf_{t \in T} \|t-u\|$.
Given two sets $T,T'$, let $d(T,T') = \inf_{t\in T, t' \in T'} \|t-t'\|$.
Finally, let the ``thickening'' $B(T,\epsilon)$ be defined as
  \[ B(T,\epsilon) = \{u \in \R' : d(T,u) < \epsilon \} . \]

\begin{assumption} \label{assume:cal}
  $\ell: \R \times \Y \to \reals^{\Y}_+$ is a loss on a finite report set $\R$, eliciting the property $\gamma: \simplex \toto \R$.
  It is embedded by $L: \reals^d \times \Y \to \reals^{\Y}_+$, which elicits the property $\Gamma: \simplex \toto \reals^d$.
  The embedding points are $\{\varphi(r) : r \in \R\}$.
\end{assumption}

Given Assumption \ref{assume:cal}, let $\mathcal{S} \subseteq 2^{\R}$ be defined as $\mathcal{S} = \{\gamma(p) : p \in \Delta_{\Y}\}$.
In other words, for each $p$, we take the set of optimal reports $R = \gamma(p) \subseteq \R$, and we add $R$ to $\mathcal{S}$.
Let $\U \subseteq 2^{\reals^d}$ be defined as $\U = \{\Gamma(p) : p \in \Delta_{\Y}\}$.
For each $U \in \U$, let $R_U = \{r: \varphi(r) \in U\}$.

The next lemma shows that if a subset of $\U$ intersect, then their corresponding report sets intersect as well.
\begin{lemma} \label{lemma:calibrated-pos}
  Let $\U' \subseteq \U$.
  If $\cap_{U\in\U'} U \neq \emptyset$ then $\cap_{U\in\U'} R_U \neq \emptyset$.
\end{lemma}
\begin{proof}
  Let $u \in \cap_{U\in\U'} U$.
  Our first claim is that there exists $r$ such that $\Gamma_u \subseteq \gamma_r$.
  This follows from Proposition \ref{prop:embed-trim}, which shows that $\trim(\Gamma) = \{ \gamma_r : r \in \R\}$.
  Each $\Gamma_u$ is either in $\trim(\Gamma)$ or is contained in some set in $\trim(\Gamma)$, by definition, proving the first claim.
  Our second claim is that $r \in \cap_{U\in\U'} R_U$, which proves the lemma.
  To prove the second claim, take any $U \in \U'$.
  There is some $p$ such that $U = \Gamma(p)$, and we have in particular $p \in \Gamma_u$.
  By the first claim, $p \in \gamma_r$.
  By definition of embedding, $p \in \gamma_r \implies \varphi(r) \in \Gamma(p) = U$, so $r \in R_U$.
\end{proof}
Lemma \ref{lemma:calibrated-pos} implies that there exists a $\psi$ such that $(L,\psi)$ indirectly elicits $\ell$: for each $u$, let $\U' = \{U\in\U : u \in U\}$ be the optimal sets that contain it; choose $r$ from the nonempty set $\cap_{U \in\U'} R_U$; and set $\psi(u) = r$.

The main problem now is to prove a ``thickened'' analogue of Lemma \ref{lemma:calibrated-pos} that extends this link to points $u$ that are up to $\epsilon$ far from an optimal set $U$.
Namely, Lemma \ref{lemma:thick-empty} will show that if $\epsilon$ is small enough, then the $\epsilon$-thickenings of all $U \in \U'$ intersect if and only if the $U$ sets themselves intersect.
Thus, if $u \in \cap_{U \in \U'} B(U,\epsilon)$, then $u \in \cap_{U \in \U'} U$, and Lemma \ref{lemma:calibrated-pos} gives some legal target report $\psi(u) = r \in \cap_{U \in \U'} R_U$.

The next few geometric results build to Lemma \ref{lemma:thick-empty}.
Then, the main proof will be completed as we have just sketched.

\begin{lemma} \label{lemma:enclose-halfspaces}
  Let $D$ be a closed, convex polyhedron in $\reals^d$.
  For any $\epsilon > 0$, there exists an \emph{open}, convex set $D'$, the intersection of a finite number of open halfspaces, such that
    \[ D \subseteq D' \subseteq B(D,\epsilon) . \]
\end{lemma}
\begin{proof}
  Let $S$ be the standard open $\epsilon$-ball $B(\{\vec{0}\},\epsilon)$.
  Note that $B(D,\epsilon) = D + S$ where $+$ is the Minkowski sum.
  Now let $S' = \{u : \|u\|_1 \leq \delta\}$ be the closed $\delta$ ball in $L_1$ norm.
  By equivalence of norms in Euclidean space~\cite[Appendix A.1.4]{boyd2004convex}, we can take $\delta$ small enough yet positive such that $S' \subseteq S$.
  By standard results, the Minkowski sum of two closed, convex polyhedra, $D'' = D + S'$ is a closed polyhedron, i.e. the intersection of a finite number of closed halfspaces. (A proof: we can form the higher-dimensional polyhedron $\{(x,y,z) : x \in D, y \in S', z = x+y\}$, then project onto the $z$ coordinates.)

  Now, if $T' \subseteq T$, then the Minkowksi sum satisfies $D + T' \subseteq D + T$.
  In particular, because $\emptyset \subseteq S' \subseteq S$, we have
    \[ D \subseteq D'' \subseteq B(D,\epsilon) . \]
  Now let $D'$ be the interior of $D''$, i.e. if $D'' = \{x : Ax \leq b\}$, then we let $D' = \{x: Ax < b\}$.
  We retain $D' \subseteq B(D,\epsilon)$.
  Further, we retain $D \subseteq D'$, because $D$ is contained in the interior of $D'' = D + S'$.
  (Proof: if $x \in D$, then for some $\gamma$, $x + B(\{\vec{0}\},\gamma) = B(x,\gamma)$ is contained in $D + S'$.)
  This proves the lemma.
\end{proof}

\begin{lemma} \label{lemma:thick-nonempty}
  Let $\{U_j : j \in \mathcal{J}\}$ be a finite collection of closed, convex sets with $\cap_{j\in\mathcal{J}} U_j \neq \emptyset$.
  Let $\delta > 0$ be given.
  Then there exists  $\epsilon_0 > 0$ such that, for all $0 < \epsilon \leq \epsilon_0$, $\cap_j B(U_j,\epsilon) \subseteq B(\cap_j U_j, \delta)$.
\end{lemma}
\begin{proof}
  We induct on $|\mathcal{J}|$.
  If $|\mathcal{J}|=1$, set $\epsilon = \delta$.
  If $|\mathcal{J}|>1$, let $j\in\mathcal{J}$ be arbitrary, let $U' = \cap_{j'\neq j} U_{j'}$, and let $C(\epsilon) = \cap_{j' \neq j} B(U_{j'},\epsilon)$.
  Let $D = U_j \cap U'$.
  We must show that $B(U_j,\epsilon) \cap C(\epsilon) \subseteq B(D,\delta)$.
  By Lemma \ref{lemma:enclose-halfspaces}, we can enclose $D$ strictly within a polyhedron $D'$, the intersection of a finite number of open halfspaces, which is itself strictly enclosed in $B(D,\delta)$.
  (For example, if $D$ is a point, then enclose it in a hypercube, which is enclosed in the ball $B(D,\delta)$.)
  We will prove that, for all small enough $\epsilon$, $B(U_j,\epsilon) \cap C(\epsilon)$ is contained in $D'$.
  This implies that it is contained in $B(D,\delta)$.

  For each halfspace defining $D'$, consider its complement $F$, a closed halfspace.
  We prove that $F \cap B(U_j,\epsilon) \cap C(\epsilon) = \emptyset$.
  Consider the intersections of $F$ with $U$ and $U'$, call them $G$ and $G'$.
  These are closed, convex sets that do not intersect (because $D$ in contained in the complement of $F$).
  So $G$ and $G'$ are separated by a nonzero distance, so $B(G,\gamma) \cap B(G',\gamma) = \emptyset$ for all small enough $\gamma$.
  And $B(G,\gamma) = F \cap B(U_j,\gamma)$ while $B(G',\gamma) = F \cap B(U',\gamma)$.
  This proves that $F \cap B(U_j,\gamma) \cap B(U',\gamma) = \emptyset$.
  By inductive assumption, $C(\epsilon) \subseteq B(U',\gamma)$ for small enough $\epsilon = \epsilon_F$.
  So $F \cap B(U_j,\gamma) \cap C(\epsilon) = \emptyset$.
  We now let $\epsilon_0$ be the minimum over these finitely many $\epsilon_F$ (one per halfspace).
\end{proof}

\begin{figure}
\caption{Illustration of a special case of the proof of Lemma \ref{lemma:thick-nonempty} where there are two sets $U_1,U_2$ and their intersection $D$ is a point. We build the polyhedron $D'$ inside $B(D,\delta)$. By considering each halfspace that defines $D'$, we then show that for small enough $\epsilon$, $B(U_1,\epsilon)$ and $B(U_2,\epsilon)$ do not intersect outside $D'$. So the intersection is contained in $D'$, so it is contained in $B(D,\delta)$.}
\includegraphics[width=0.24\textwidth]{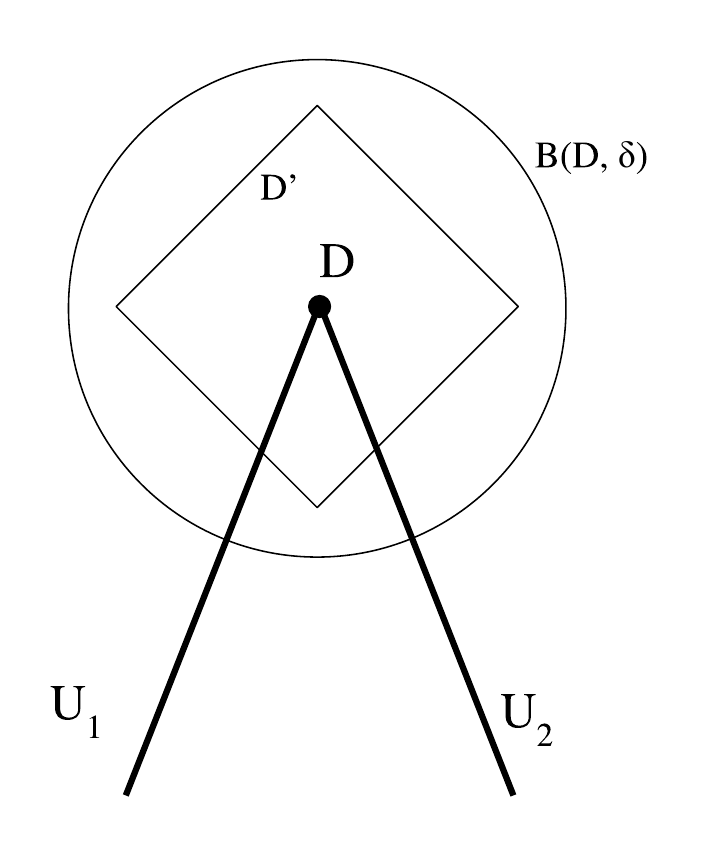} \hfill
\includegraphics[width=0.24\textwidth]{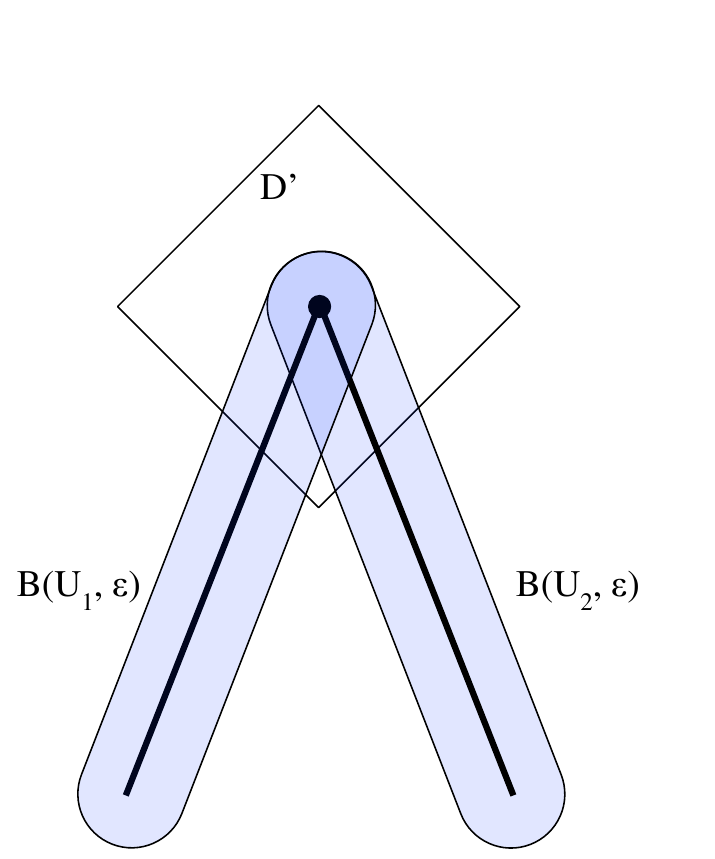} \hfill
\includegraphics[width=0.24\textwidth]{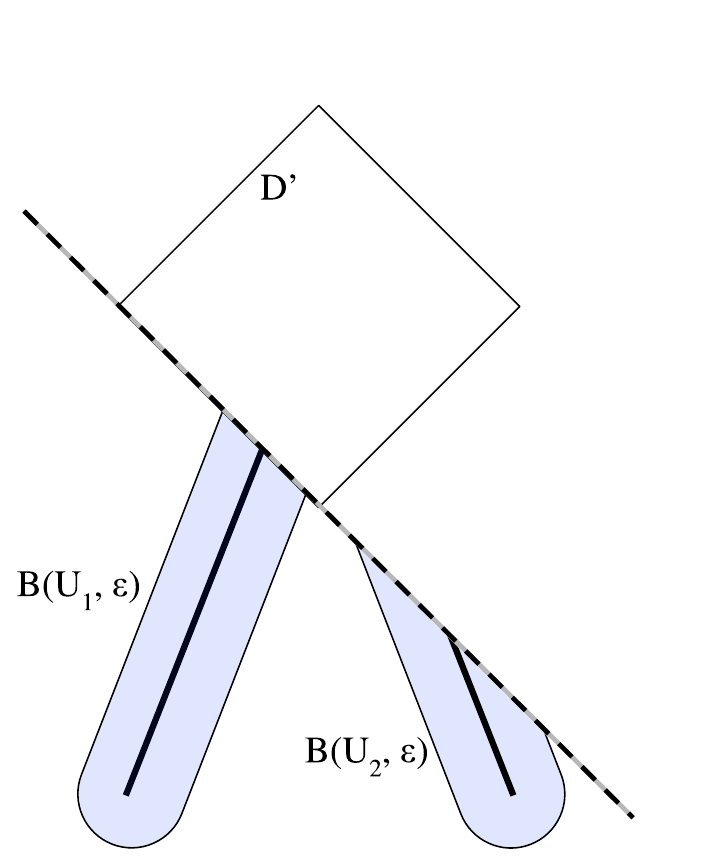} \hfill
\includegraphics[width=0.24\textwidth]{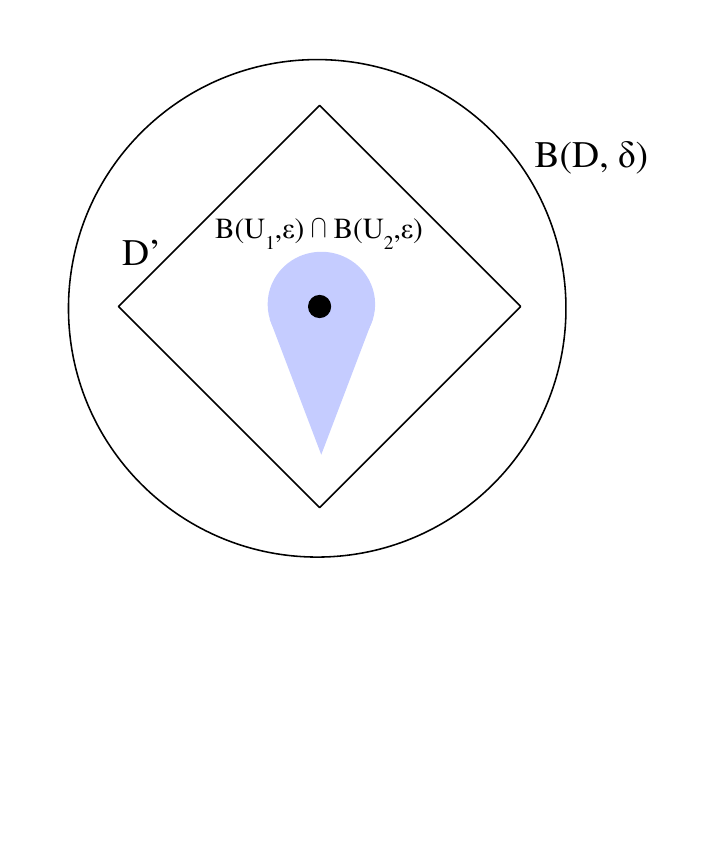}
\end{figure}

\begin{lemma} \label{lemma:thick-empty}
  Let $\{U_j : j \in \mathcal{J}\}$ be a finite collection of nonempty closed, convex sets with $\cap_{j\in\mathcal{J}} U_j = \emptyset$.
  Then there exists  $\epsilon_0 > 0$ such that, for all $0 < \epsilon \leq \epsilon_0$, $\cap_{j\in\mathcal{J}} B(U_j,\epsilon) = \emptyset$.
\end{lemma}
\begin{proof}
  By induction on the size of the family.
  Note that the family must have size at least two.
  Let $U_j$ be any set in the family and let $U' = \cap_{j' \neq j} U_{j'}$.
  There are two possibilities.

  The first possibility, which includes the base case where the size of the family is two, is the case $U'$ is nonempty.
  Because $U_j$ and $U'$ are non-intersecting closed convex sets, they are separated by some distance $\delta$.
  So $B(U_j, \delta/3) \cap B(U', \delta/3) = \emptyset$.
  By Lemma \ref{lemma:thick-nonempty}, there exists $\epsilon'_0 > 0$ such that $\cap_{j'\neq j} B(U_{j'},\epsilon) \subseteq B(U', \delta/3)$ for all $0 < \epsilon \leq \epsilon'_0$.
  Pick $\epsilon_0 = \min\{\epsilon'_0,\delta/3\}$.
  Then for all $0 < \epsilon \leq \epsilon_0$, the intersection of $\epsilon$-thickenings is contained in the $(\delta/3)$-thickening of the intersection, which is disjoint from the $(\delta/3)$-thickening of $U_j$, which contains the $\epsilon$-thickening of $U_j$.

  The second possibility is that $U'$ is empty.
  This implies we are not in the base case, as the family must have three or more sets.
  By inductive assumption, for all small enough $\epsilon$ we have $\cap_{j' \neq j} B(U_{j'},\epsilon) = \emptyset$, which proves this case.
\end{proof}

\begin{corollary} \label{cor:thick-intersect}
  There exists $\epsilon_0 > 0$ such that, for any $0 < \epsilon \leq \epsilon_0$, for any subset $\{U_j : j \in \mathcal{J}\}$ of $\U$, if $\cap_j U_j = \emptyset$, then $\cap_j B(U_j,\epsilon) = \emptyset$.
\end{corollary}
\begin{proof}
  For each subset, Lemma \ref{lemma:thick-empty} gives an $\epsilon_0 > 0$.
  We take the minimum over these finitely many subsets of $\U$.
\end{proof}

\begin{theorem} \label{thm:small-eps-thick}
  For all small enough $\epsilon$, the epsilon-thickened link $\psi$ (Construction \ref{const:eps-thick-link}) is a well-defined link function from $\R'$ to $\R$, i.e. $\psi(u) \neq \bot$ for all $u$.
\end{theorem}
\begin{proof}
  Fix a small enough $\epsilon$ as promised by Corollary \ref{cor:thick-intersect}.
  Consider any $u \in \R'$.
  If $u$ is not in $B(U,\epsilon)$ for any $U \in \U$, then we have $\Psi(u) = \R$, so it is nonempty.
  Otherwise, let $\{U_j : j \in \mathcal{J}\}$ be the family whose thickenings intersect at $u$.
  By Corollary \ref{cor:thick-intersect}, because of our choice of $\epsilon$, the family themselves has nonempty intersection.
  By Lemma \ref{lemma:calibrated-pos}, their corresponding report sets $\{R_j : j \in \mathcal{J}\}$ also intersect at some $r$, so $\Psi(u)$ is nonempty.
\end{proof}

Theorem \ref{thm:thickened-separated}, which we restate here, is now almost immediate.
\thickenedseparated*
\begin{proof}
  We create $\psi$ using Construction \ref{const:eps-thick-link} with the $L_{\infty}$ norm.
  By Theorem \ref{thm:small-eps-thick}, for all small enough $\epsilon$, $\psi$ is well-defined everywhere.

  To prove separation, suppose $u$ and $p$ are given such that $d_{\infty}(u,U) \leq \epsilon$, where $U = \Gamma(p)$.
  Then in Construction \ref{const:eps-thick-link}, $\psi(u) \in \Psi(u) \subseteq R_U = \{r : \varphi(r) \in U\}$.
  By definition of embedding, $\varphi(r) \in U = \Gamma(p) \implies r \in \gamma(p)$.
  So we obtain $\psi(u) \in \gamma(p)$ whenever $d_{\infty}(u,\Gamma(p)) \leq \epsilon$, which proves $\epsilon$-separation of the link $\psi$.
\end{proof}

\section{General characteristics of polyhedra}
\subsection{Definitions and preliminaries}
\begin{definition}[Closed halfspace]
	For any $(w,b) \in \reals^d \times \reals$, let $H_{(w,b)}^+ := \{ x \in \reals^d \mid \inprod{x}{w} \geq b\}$ be the closed halfspace defined by $(w,b)$. 
\end{definition}
\begin{definition}[Hyperplane]
	For any $(w,b)\in\reals^d \times \reals$, let $H_{(w,b)} := \{ x \in \reals^d \mid \inprod{x}{w} = b\}$ be the hyperplane generated by $(w,b)$.
\end{definition} 
Observe that the hyperplane $H_{(w,b)}$ is the boundary of $H^+_{(w,b)}$.

\begin{definition}[Polyhedron - halfspace representation]
	A \emph{polyhedron} $P$ is an intersection of a finite set of closed halfspaces $\H = \{H^+_{(w_i,b_i)}\}_{i=1}^k$ presented in the form $P = \cap \H$.
\end{definition}
Observe that by the halfspace representation, a polyhedron need not be bounded.

\begin{definition}[Supports]
	A hyperplane $H$ \emph{supports} the polyhedron $P$ if 
	(i) $P \subseteq H^+$ or $P \subseteq -H^+$, and
	(ii)$H \cap \mathrm{bd}(P) \neq \emptyset$.
	Moreover, $H$ supports $P$ at $x$ if $x \in H \cap \mathrm{bd}(P)$.
\end{definition}

\begin{definition}[Face, facet]
  Let $P \subseteq \reals^d$ be a convex polyhedron.
  A halfspace $H^+_{(w,b)}$ is \emph{valid} for $P$ if $P \subseteq H^+_{(w,b)}$.
  A \emph{face} $F_{(w,b)}$ of the polytope $P$ is any set of the form
  \begin{equation*}
    F_{(w,b)} = P \cap H_{(w,b)}~,~
  \end{equation*}
  for any valid halfspace $H^+_{(w,b)}$.
  The dimension of a face $F$ is the dimension of its affine hull $\dim(F) := \dim(\affhull(F))$.
  A face $F$ with $\dim(F) = \dim(\affhull(P)) - 1$ is called a facet.
\end{definition}

\begin{claim}
  A face $F_{(w,b)}$ of the polyhedron $P$ 
  is nonempty if and only if $H_{(w,b)}$ is a supporting hyperplane of $P$.
\end{claim}

\begin{theorem}\label{thm:polyhedron-uniquely-gen-facets}
  Given a $d$-dimensional polyhedron $P \subseteq \reals^d$, 
  (i) there is a unique finite set of closed halfspaces $\H^*$ such that $P = \cap \H^*$, 
  (ii) for all finite sets of closed halfspaces $\H$ such that $P = \cap \H$, we have $\H^* \subseteq \H$, and
  (iii) $\{H_{(w,b)}\cap P \mid H^+_{(w,b)} \in \H^*\}$ is the set of facets of $P$.
\end{theorem}
\begin{proof}
	Since $P$ is $d$-dimensional in $\reals^d$, it therefore has nonempty interior.	
	\citet[Proposition 4.5(i)]{gallier2008notes} can then be applied to yield the existence of a unique finite set of closed halfspaces $\H^*$ uniquely determining $P$, allowing us to conclude (i).
	Additionally, (iii) is shown by~\citep[Proposition 4.5(ii)]{gallier2008notes}.
	
	It is just left to show (ii).
	By (i), we know that each $H_{(w,b)} \in \H^*$ uniquely determines a facet of $P$.
	Thus, $F_{(w,b)}$ is a facet of $P$, and is of dimension $d-1$.
	This implies the facet $F_{(w,b)}$ can be defined by $d$ affinely independent points (contained in $P$), whose affine hull is $H_{(w,b)}$.
	As halfspaces are uniquely determined, so is the facet $F_{(w,b)} = P \cap H_{(w,b)}$.
	As polyhedron are uniquely determined by their facets (by Minkowski's uniqueness theorem), we must have $H_{(w,b)} \in \H^*$.

\end{proof}

\subsection{Notation}
Within this appendix, we use some self-contained notation.
We will later consider losses over a finite set of outcomes $\Y$; to make notation consistent, we use $\reals^\Y_+$ throughout, and let $d := |\Y|+1$.

Fix a set $\V \subseteq \reals^\Y_+$, and consider the concave function $g_\V : x \mapsto \inf_{v \in \V}\inprod{v}{x} - \delta(x \mid \reals^\Y_+)$.
We denote the hypograph of $g_\V$ by $\hyp(g_\V) = \{(x,c) \mid c \leq g(x)\} \subseteq \reals^\Y_+ \times \reals$.

Given any $v$ in a given set $\V$, define $H^+_v := H^+_{(v, -1)}$, where $(v, -1)\in\reals^\Y_+\times\reals$ defines a halfspace.
Similarly, we denote $H_y^+ := H_{(e_y, 0)}^+$ for any $y \in \Y$; the latter will help us restrict our $\hyp(g_\V)$ to the nonnegative orthant. 
Similarly, we let $H_v := H_{(v,-1)}$ for $v \in \reals^\Y_+$ and define $H_y := H_{(e_y, 0)}$. 

Finally, given a set $\V \subseteq \reals^d$, we let $\H_{\V} = \{H_v^+ \mid v\in\V\}$ denote the set of halfspaces generated by $\V$, $\H_\Y = \{H_y^+ \mid y\in\Y\}$.
If $\V$ and $\Y$ are understood from context, we may denote $\H := \H_\V \cup\H_\Y$.

\subsection{Constructing a unique minimum set of facets from a finitely generated polyhedron}\label{appsubsec:phase1}
Suppose $\V$ is a finite set $\V \subset \reals_+^\Y$.
Throughout, we will work with a function $g := g_\V$ generated by $\V$ of the following form.

\begin{definition}\label{def:g-finite}
  Define the function $g_\V : \reals_+^\Y \to \reals_+$ by
  \begin{align*}
    g_\V(x) = \min_{v\in\V} \inprod{x}{v} - \delta(x \mid \reals_+^\Y)~,~
  \end{align*}
\end{definition}

First, we observe that the region generated by the intersection of the $\H_y$ halfspaces restricts the hypograph to the nonnegative orthant.
\begin{claim}\label{claim:x-nonneg-orthant-iff-intersection-HY}
  $\cap \H_\Y = \reals^\Y_+ \times \reals$.
\end{claim}
\begin{proof}
  The result follows if we show $x \in \reals^\Y_+ \iff (x,c) \in \cap \H_\Y$ for all $c \in \reals$.

  $\implies$
  Fix any $c \in \reals$.
  $x \in \reals^\Y_+ \iff x_y \geq 0$ for all $y \in \Y$.
  This means that for any $y \in \Y$, $(x,c) \in \{(x,c) \mid x_y \geq 0\} = H_y$.
  As $y$ and $c$ were arbitrary, this shows the forward direction.
  
  $\impliedby$
  $(x,c) \in \cap \H_\Y$ implies $x_y \geq 0$ for all $y \in \Y$, and therefore $x \in \reals^d_+$.	
\end{proof}

Throughout this section let $G:= \hyp(g_\V)$ for a fixed finite set $\V$.
Now, we can define $G$ as the intersection of halfspaces generated by $\V$ on the nonnegative orthant.
\begin{claim}\label{claim:G-intersection-H}
  Given a finite set $\V \subset \reals^\Y_+$, define $\H = \H_\V \cup \H_\Y$ and $G := \hyp(g_\V)$.
  Then $G = \cap \H$.
\end{claim}
\begin{proof}
  $(x,c) \in G \iff g(x) - c \geq 0 \iff \min_{v \in \V}\inprod{v}{x} - c \geq 0$ and $x \in \reals_+^\Y$, which is true if and only if $\inprod{v}{x} - c \geq 0 \,\forall v \in \V$ and $x_y \geq 0$ for all $y$.
  In turn, this statement holds if and only if $(x,c) \in H^+_v$ for all $v \in \V$ and in $H^+_y$ for all $y \in \Y$ respectively, so $(x,c) \in \cap \H$.
\end{proof}

We proceed with some observations about facets and dimension of $G$.

\begin{claim}\label{claim:G-full-dimensional}
  Given a finite, nonempty set $\V \subset \reals^d$ and $G := \hyp(g_\V)$, $G$ is $d$-dimensional.  
\end{claim}
\begin{proof}
  Since $g_\V$ is nonnegative on $\reals_+^\Y$, $G$ therefore contains $\{(x,c) \mid x\in\reals^\Y_+, c \leq 0\}$, which is $(|\Y| + 1)$-dimensional.
\end{proof}

This result, in conjunction with Theorem~\ref{thm:polyhedron-uniquely-gen-facets} yields a unique set of halfspaces $\H^*$ generating $\hyp(g_\V)$.

\begin{lemma}\label{lem:G-unique-facets-Hstar}
  Given a finite set $\V \subset \reals^\Y_+$, define $\H = \H_\V \cup \H_\Y$ and $G := \hyp(g_\V)$.
  There is some unique $\H^* \subseteq \H$ such that $G = \cap \H^*$. 
  Moreover, for each $H_{(w_i, b_i)} \in\H^*$, the face $F_{(w_i, b_i)} = H_{(w_i, b_i)} \cap G$ is a facet.
\end{lemma}
\begin{proof}
  Since $G$ is full-dimensional by Claim~\ref{claim:G-full-dimensional}, this follows immediately from Theorem~\ref{thm:polyhedron-uniquely-gen-facets}.
\end{proof}
We can also show that the set $\H_\Y$ is contained in $\H^*$ so that we can separate the facets generated by $\H^*$ into a partition of vertical and non-vertical facets of $G$.

\begin{lemma}\label{lem:HY-subset-Hstar}
  Given a finite set $\V \subset \reals^\Y_+$, consider the unique set of facet-defining halfspaces $\H^* \subseteq (\H_\V \cup \H_\Y)$ such that $\hyp(g_\V) = \H^*$.
  Then $\H_\Y \subseteq \H^*$.
\end{lemma}
\begin{proof}
  If there was a $y \in \Y$ such that $H_y = \{(x,c) \mid x_y \geq 0\}$ was not in $\H^*$, then we would either have some $c_1 > 0$ such that $\{(x,c) \mid x_y \geq c_1\} \in \H^*$, or we have a point $x$ such that $x_y < 0$ but $g(x) > -\infty$.
  The first cannot happen as we take $g$ to be defined (and finite) on $x = e_y \in \reals^\Y_+$ and is concave.
  Moreover, the second cannot be true by construction of $g$ including the $0-\infty$ indicator on $\reals^\Y_+$.	
\end{proof}

\begin{corollary}\label{cor:unique-set-loss-vectors-defining-facets}
  Suppose we are given a finite set $\V \subset \reals^\Y_+$, and consider the smallest (in cardinality) unique set $\H^*$ such that $\hyp(g_\V) =: G = \cap \H^*$.
  There is a unique finite set $\V^* \subseteq \reals^\Y_+$ such that $\H^* = \H_\Y \cup \H_{\V^*}$.
  Moreover, $F_v$ is a facet of $G$ for each $v\in\V^*$.
\end{corollary}
\begin{proof}
  Since $G$ is full-dimensional, all facets of $G$ are uniquely determined by the hyperplanes $H$ whose halfspaces $H^+$ compose $\H^*$ by Lemma~\ref{lem:G-unique-facets-Hstar}.
  Any facet $F$ must then be some intersection of an $H_y \cap G$ or $H_v \cap G$.
  Take $\H_{\V^*} := \H^* \setminus \H_{\Y}$, and $\V^*$ to be the unique set generating $\H_{\V^*}$.
  (Uniqueness of $\V^*$ follows from uniqueness of $\H^*$.)
  The moreover follows since every $H_v \in \H_{\V^*} \subseteq \H^*$ generates a facet.

\end{proof}

\begin{corollary}\label{cor:anything-gen-G-subset-Hstar}
  Let $\V \subset \reals^\Y_+$ be a finite set such that $\H := \H_\Y \cup \H_{\V}$ satisfies $\hyp(g_{\V}) = \cap \H$, and take $\V^*$ the unique set such that $\H^* = \H_\Y \cup \H_{\V^*}$.  
  Then $\V^* \subseteq \V$.
\end{corollary}
\begin{proof}
  $\H_\Y \cup \H_{\V^*} = \H^*$ by Corollary~\ref{cor:unique-set-loss-vectors-defining-facets}, and $\H^* \subseteq \H = \H_\Y \cup \H_{\V'}$ by Lemma~\ref{lem:HY-subset-Hstar}, ergo $\V^* \subseteq \V'$.
\end{proof}

We now show that we can equivalently construct $g_\V$ through the unique finite set $\V^*$ instead of the given set of vectors $\V$.
\begin{lemma}\label{claim:g-gen-by-Vstar}
	Given a finite set $\V$, consider the smallest (in cardinality) set $\V^* \subseteq \V$ such that $\hyp(g_\V) = \hyp(g_{\V^*})$.
  $g_\V(x) = \min_{v \in \V^*}\inprod{v}{x} - \delta(x \mid \reals_+^\Y) = g_{\V^*}(x)$
\end{lemma}
\begin{proof}
  $\hyp(g_\V) = \cap (\H_\Y \cup \H_\V)= \cap \H^* = \cap (\H_\Y \cup \H_{\V^*}) = \{(x,c) \in \reals_+^\Y \times \reals \mid \inprod{v^*}{x} \geq c$ for all $v^* \in \V^* \}$ where the first equality follows as $\H^* \subseteq \H$.
  Since $G$ is the hypograph of $g_\V$, this means $g_\V$ can be written as 
  $g_\V(x) = \min_{v \in \V^*}\inprod{v}{x} - \delta(x \mid \reals^\Y_+) = g_{\V^*}(x)$. 
\end{proof}

This series of results will let us reason about the Bayes risk (and $1$-homogeneous extension) of losses with finite representative sets in \S~\ref{subsec:lem-X-proof}.

\subsection{Infinitely generated polyhedron}\label{appsubsec:phase2}

Now suppose $L$ is a minimizable loss function.
For $x\in\reals^\Y_+$, consider the $1$-homogeneous extension of Bayes risk  $\risk{L}_+(x) := \inf_{r\in\R} \inprod{x}{L(r)}$, which we assume is polyhedral throughout.
Consider the function $g(x) = \risk{L}_+(x)$.
We now define $\V = L(\R) \subseteq \reals^\Y_+$ and $\H_{\V} = \{H_v^+ \mid v\in\V\}$.  
Unlike above, these could be infinite sets.  
Now let $\H = \H_\Y \cup \H_{\V}$; again, this may be infinitely generated.
\begin{claim}
	$\hyp(g_\V) = \cap\H$.
\end{claim}
\begin{proof}
	Observe that $x \in \reals^\Y_+ \iff (x, c) \in \cap \H_\Y$.
	Let $x \in \reals^\Y_+$.
	\begin{align*}
	(x,c) \in \hyp(g_\V)
	&\iff g_\V(x) \geq c & \text{ Definition of hypograph}\\
	&\iff \risk L_+(x) \geq c & \text{ As $g_\V=\risk L_+$}\\
	&\iff \inprod{v}{x} \geq c \,\, \forall v \in \V & \text{By def of $\risk L_+$ as the infimum over $v \in \V$ of}\\
	& & \text{the inner product with $x$ and minimizable.} \\
	&\iff (x,c) \in H^+_{v} \,\, \forall v \in L(\R) & \text{By definition of each halfspace}\\
	&\iff (x,c) \in \cap \H_\V & \text{Since true for all $v \in \V$} \\ 
	\end{align*}
	Combining the two equalities (e.g., $\cap \H = (\cap \H_\Y) \cap (\cap \H_\V)$), we have $\hyp(g_\V) = \cap \H$.
	
\end{proof}

If follows that $\hyp(g_\V)$ is a polyhedron, and $d+1$-dimensional as $\{(x,c) \mid x \in \reals^\Y_+, c \leq 0\} \subseteq \hyp(g_\V)$.
Thus, we can apply Theorem~\ref{thm:polyhedron-uniquely-gen-facets} to conclude that $\hyp(g_\V)$ has a unique minimum halfspace representation $\cap \H^*$, where $\H^* = \H_\Y \cup \H_{\V^*}$ for the unique finite set $\V^*$.
Moreover, we have $g_\V = g_{\V^*}$, as $\H^*$ is the unique minimum halfspace representation for $\hyp(g_\V)$.

\begin{corollary}\label{cor:finite-rep-set}
	Given a minimizable loss $L : \R \times \Y \to \reals$ with polyhedral extended Bayes risk $\risk L_+$, take $L(\R) = \V$ and $\V^*$ to be the unique smallest (in cardinality) set such that $\hyp(g_{\V}) = \cap (\H_\Y \cup \H_{\V^*})$.
	There is a finite set $\R^* \subseteq \R$ such that $L(\R^*) = \V^*$ (without duplicates).
\end{corollary}

\subsection{Projecting from $\reals^d_+$ to $\reals^\Y_+$}\label{subsec:project-pi}
Throughout this section, suppose that we are given a finite set $\V$ such that $\hyp(g_\V)$ is a polyhedron. 
Moreover, we denote the unique smallest (in cardinality; finite) set $\V^* \subseteq \V$ such that $g_\V = g_{\V^*}$.

We now consider projections from $\reals^\Y_+ \times \reals$ onto the positive orthant $\reals^\Y_+$.

\begin{claim}\label{claim:vstar-supporting-G}
	For all $x\in\reals^\Y_+$, there exists $v^*\in\V^*$ such that $H_{v^*}$ supports $G := \hyp(g_\V)$ at $(x,g_\V(x))$.
\end{claim}
\begin{proof}
	By Claim~\ref{claim:g-gen-by-Vstar} and Corollary~\ref{cor:finite-rep-set}, we know that $g_\V(x) = g_{\V^*}(x) = \min_{v \in \V^*}\inprod{v}{x}$.
	In particular, consider the normal $v^* \in \argmin_{v \in \V^*}\inprod{v}{x}$; we claim that $H_{v^*}$ supports $G$ at $(x,g(x))$.
	First, $G \subseteq H^+_{v^*}$ by definition of $G$ as the intersection of halfspaces including $H^+_{v^*}$.
	Thus, it is just left to show that $(x, \inprod{v^*}{x}) \in H_{v^*} \cap G$.
	$H_{v^*} = \{(x,c) \mid \inprod{v^*}{x} = c\}$.
	By definition of $g$, we have $(x,g(x)) = \inprod{v^*}{x}$, so $(x,g(x)) \in H_{v^*}$.
	Moreover, $(x,g(x)) \in G = \{(x,c) \mid g(x) \geq c\}$ trivially since $g(x) \geq g(x)$.
\end{proof}

Define the projection $\pi:\reals^\Y\times \reals \to \reals^\Y, (x,c) \mapsto x$.
This projected faces generated by $\V^*$ covers the nonnegative orthant.
\begin{corollary}\label{cor:projected-Vstar-faces-cover-pos-orthant}
	$\cup_{v\in\V^*} \pi(F_v) = \reals^\Y_+$.
\end{corollary}

Moreover, the projection $\pi$ preserves dimension of faces.

\begin{claim}\label{claim:pi-preserves-dim}
	For all $v\in\reals^\Y_+$, $\dim(F_v) = \dim(\pi(F_v))$.
\end{claim}
\begin{proof}
	As a reminder, we define the dimension of a polytope to be the dimension of its affine hull.
	Suppose we are given $|\Y|+1$ affinely independent vectors $z_i$ in $F_v$. 
	We claim their projections $\{\pi(z_i)\}$ are affinely independent.
	Let $a_1 + \ldots + a_{\Y+1} = 0$, such that $\sum_i a_i \pi(z_i) = 0$.
	We want to conclude that we must have $a_i = 0$ for all $i$, meaning they are affinely independent.
	
	Observe $z_i = (x_i, \inprod{v}{x_i})$ for all $i$; therefore, if $z_i \in F_v$ (e.g., $F_v$ supports $\hyp(g_\V)$ at $(x, \inprod{v}{x})$), then we also have $z_i \in H_v$.
	So $0 = \sum_i a_i \pi(z_i) = \sum_i a_i x_i$.
	Moreover, the sum $\sum_i a_i z_i = \sum_i a_i (x_i, \inprod{v}{x_i}) = (\sum_i a_i x_i, \inprod{v}{\sum_i a_i x_i}) = (\vec 0,0) = \vec 0$.
	Thus, since $a_i = 0$ for all $i$, the set $\{z_i\}$ is affinely independent and the dimensions of the affine hulls are therefore equal.
\end{proof}

Since we preserve the dimension of these projected spaces, we can now study equivalence of projected faces of the hypograph and regions of support of $g$ for any $v \in \V$.

\begin{lemma}\label{lem:projected-faces-iff-support-iff-argmin}
	Fix any $x \in \reals^\Y_+$.
	For any $v \in \V$, that the following are equivalent:

	(1) $(x,g_\V(x)) \in F_v$
	
	(2) $\inprod{v}{x}= g_\V(x)$
	
	(3) $v \in \argmin_{v' \in \V} \inprod{v'}{x}$
	
	(4) $x \in \pi(F_v)$
\end{lemma}
\begin{proof}
	\begin{align*}
	(1) \quad \quad (x,g(x)) \in F_v
	&\iff (x,g_\V(x)) \in \{(x',c) \in \hyp(g_\V) \mid \inprod{v}{x'} = c\} & \\
	&\iff  \inprod{v}{x} = g_\V(x) & (2)\\
	&\iff \inprod{v}{x} = \min_{v' \in \V}\inprod{v'}{x} & \\
	&\iff v \in \argmin_{v' \in \V}\inprod{v'}{x} & (3)
	\end{align*}
	This covers $1 \iff 2 \iff 3$.
	
	For $1 \iff 4$, the forward implication follows trivially by applying the definition of the projection $\pi$.
	For the reverse implication, consider some $x \in \pi(F_v)$.
	There must be a $c \in \reals$ so that $(x,c) \in F_v$.
	Expanding, this is actually saying $(x,c) \in \{(x',c') \in \hyp(g_\V) \mid \inprod{v}{x'} = c\}$.
	In particular, this is true when $c = \inprod{v}{x}$, which defines a face of $\hyp(g_\V)$ at $x$ if any only if $\inprod{v}{x} = g_\V(x)$.
	Therefore, we have $(x, g_\V(x)) \in F_v$.
\end{proof}

Now we can observe a set of normals $\V'$ generating faces whose projections cover $\reals^\Y_+$ if and only if the set contains $\V^*$.
This will translate to a set being representative for a loss if and only if it contains a finite minimum representative set (in settings where one exists.)
\begin{claim}\label{claim:Vprime-projected-faces-cover-iff-Vstar-subset-Vprime}
	For $\V' \subseteq \V$, we have
	$\cup_{v\in\V'} \pi(F_v) = \reals^\Y_+ \iff \V^* \subseteq \V'$.
\end{claim}
\begin{proof}
	We would like to show
	$\cup_{v\in\V'} \pi(F_v) = \reals^\Y_+ \iff \cap(\H_\Y \cup \H_{\V'}) = \hyp(g_\V) := G$ then apply Corollary~\ref{cor:anything-gen-G-subset-Hstar}.

	$\implies$ 
	Suppose $\cup_{v\in\V'} \pi(F_v) = \reals^\Y_+$.
	Fix $x \in \reals^\Y_+$.
	Since $\V' \subseteq \V$, and $G = \cap (\H_\Y \cup \H_\V)$, we immediately have $G \subseteq (\H_\Y \cup \H_{\V'})$.
	Therefore, we just need to show the other direction of inclusion.
	$(x,c) \in \cap \H_\Y$ for all $c \in \reals$ by Claim~\ref{claim:x-nonneg-orthant-iff-intersection-HY}, so it is left to show that $(x,c) \in \cap \H_{\V'}$, which yields the desired result.
	First, $(x,c) \in \cap \H_{\V'} \implies (x,c) \in H^+_{v'}$ for all $v' \in \V'$.
	This implies $c \leq \inprod{v'}{x}$ for all $v' \in \V'$.
	By Lemma~\ref{lem:projected-faces-iff-support-iff-argmin}, there exists a $v' \in \V'$ such that $(x,g(x)) \in F_{v'}$, and $\inprod{v'}{x} = g(x)$.
	Therefore, we have $c \leq \inprod{v'}{x} = g(x)$, and $(x,c) \in G$ follows since $G$ is the hypograph of $G$.

	$\impliedby$
	$G = \cap (\H_\Y \cup \H_{\V'})$ if and only if $G = \cap \H_{\V'}$ when restricting to $x \in \reals_+^\Y$.
	This implies that for all $x \in \reals_+^\Y$, there exists a $v' \in \V'$ such that $(x,g(x)) \in H_{v'} \cap G = F_{v'} \implies x \in \pi(F_{v'})$.
	As this is true for all $x \in \reals_+^\Y$, we have $\cup_{v \in \V'} \pi(F_{v}) = \reals_+^\Y$.
	
\end{proof}

For the set of normals $\V^*$ generating facets of $\hyp(g_\V)$, we have each of these projected facets being full-dimensional in the projected space as well.
\begin{claim}\label{claim:projected-Vstar-faces-full-dim}
	For all $v \in \V^*$, $\pi(F_v)$ is full dimensional in $\reals_+^\Y$.
\end{claim}
\begin{proof}
	By Corollary~\ref{cor:unique-set-loss-vectors-defining-facets}, $F_v$ is a facet of $\hyp(g_\V)$ in $\reals^d_+$, meaning it is $(d - 1)$-dimensional.
	Moreover, Claim~\ref{claim:pi-preserves-dim} states that the dimension of $F_v$ is preserved for each $v \in \V^*$.
	Thus, $\dim(F_v) = \dim(\pi(F_v)) = |\Y|$.
\end{proof}

Denote $\Lambda := \{\pi(F_v) \mid v \in \V^*\}$ as the set of projected facets generated by $\V^*$.

\begin{claim}\label{claim:Vprime-projected-faces-cover-iffprojected-faces-subsets}
	Suppose we are given a minimizable loss $L$ will polyhedral extended Bayes risk $\risk{L}_+(x) := \inf_{r \in \R} \inprod{x}{L(r)}$.
	Take $\R' \subseteq \R$ with $\V' := L(\R')$. 
	For $\V' \subseteq \V := L(\R)$, we have $\cup_{v \in \V'} \pi(F_v) = \reals_+^\Y \iff \Lambda \subseteq \{ \pi(F_v) \mid v \in \V'\}$.
\end{claim}
\begin{proof}
	$\implies$
	The result follows if $\V^* \subseteq \V'$, which is exactly the forward implication of Claim~\ref{claim:Vprime-projected-faces-cover-iff-Vstar-subset-Vprime}.
	Explicitly, for all $v \in \V^*$ we also have $v \in \V'$, so, $\pi(F_v) \in \Lambda \cap \{\pi(F_v) \mid v \in \V'\} = \Lambda$.
	
	$\impliedby$
	If $\Lambda \subseteq \{ \pi(F_v) \mid v \in L(\R') \}$, then $\cup_{\lambda \in \Lambda} \lambda \subseteq \cup_{v \in \V'} \pi(F_v)$.
	By Corollary~\ref{cor:projected-Vstar-faces-cover-pos-orthant}, we have $\cup_{\lambda \in \Lambda} \lambda = \reals_+^\Y$, so $\reals_+^\Y \subseteq \cup_{v \in \V'} \pi(F_v)$.
	The other direction of subset inequality following from $\hyp(g_\V)$ being finite only on $\reals_+^\Y$.	
\end{proof}

Given a loss $L$ with $\risk L_+$ polyhedral and $\V = L(\R)$, take $v = L(r)$ for any $r \in \R$.
As $H_{v}^+ \in \H$, it supports $\hyp(g_\V)$ or it is redundant in $\H$ (or both).

\begin{claim}\label{claim:faces-contained-in-facets}
	Given $\V \subseteq \reals^\Y$ satisfying the above requirements, if a $H_v \in \H$ supports $\hyp(g_\V) =: G$, then $F_{v} = H_{v} \cap G$ is a nonempty face of $G$, and thus a subset of a facet. 
\end{claim}
\begin{proof}
	Since $H_v$ supports $G$, we know that $G \subseteq H_v^+$ and $F_v$ is not empty.
	Moreover, $H^+_v$ is valid, and we have $F_v = F_{(v,-1)}$ is a face of $P$ by definition.
	
	Moreover, the faces of $G$ are all convex polyhedra.
	Any face of $G$ is must then be a lower-dimensional face of a facet, and therefore a subset.
\end{proof}

\begin{claim}\label{claim:vface-subset-some-vstar-face}
	For any $v \in \V$, the face $F_{v} \subseteq F_{v^*}$ for some $v^* \in \V^*$. 
\end{claim}
\begin{proof}
	As $\hyp(g_{\V^*})$ is polyhedral, each of its faces are convex polyhedra, and is also a face of some facet of $\hyp(g_{\V^*})$; these facets are defined by $\V^*$ and $\Y$ (Corollary 6).
	
	It then suffices to show that if $F_v$ is a face of the facet $F_y$ for some $y\in \Y$, then it must also be a face of $F_{v^*}$ for a $v \in \V^*$; equivalently, $F_v$ is not a facet, and thus $F_v \neq F_y$.
	Recall from the definition of a face that $F_v = G \cap H_{(v,-1)}$ and $F_y = G \cap H_{(e_y, 0)}$.
	As facets are (uniquely) determined by halfspaces, and $F_y$ is a facet, we either have $F_v \neq F_y$ (in particular, if $F_v$ is a facet) or $F_v \subsetneq F_y$ (if $F_v$ is not a facet).
	In both cases, we have $F_v \neq F_y$, and the result follows.
	
\end{proof}

\begin{corollary}\label{cor:projected-faces-subset-projected-facet}
	For $v,v^*$ such that $F_v \subseteq F_{v^*}$ and $v^*\in \V^*$, $\pi(F_{v}) \subseteq \pi(F_{v^*})$. 
\end{corollary}

Now, we can conclude that projected facets generated by $\V$ contain all other projected faces of $G$.
\begin{corollary}\label{cor:exists-vstar-projected-face-subset}
	For any $v \in \V$, there is a $v^* \in \V^*$ such that $\pi(F_v) \subseteq \pi(F_{v^*})$.
\end{corollary}
\begin{proof}
	This is exactly Claim~\ref{claim:vface-subset-some-vstar-face} and Corollary~\ref{cor:projected-faces-subset-projected-facet} chained together.
\end{proof}

\subsection{Translating to properties: projecting from $\reals^\Y_+$ to $\simplex$}\label{subsec:project-f}

Let $f_\V:\reals^\Y\to\reals_+\cup\{-\infty\}$ be a polyhedral concave function with $\dom(f_\V) = \simplex$.
\begin{claim}\label{claim:f-min-affine}
	We may write $f_\V(p) = \min_{v\in\V} \inprod{p}{v} + \delta(p \mid \simplex)$ for some finite set $\V \subset \reals^\Y_+$. 
\end{claim}
\begin{proof}
  We will think of $f_\V$ as defined $f_\V:\reals^\Y\to\reals_+\cup\{-\infty\}$  with $\dom(f_\V) = \simplex$.
  For $p \in \simplex$, we know $\sum_i p_i = 1$, and can write any inner product $\inprod{p}{b} - \beta = \inprod{p}{b} - \inprod{p}{\beta \ones} = \inprod{p}{b - \beta \ones}$.
  If $p \not \in \simplex$, then $f_\V(p) = -\infty$ and inner products are not used to compute $F$.
  Moreover, since $f_\V$ is polyhedral, it is finitely generated \citep[Proposition 19.1.2]{rockafellar1997convex} and can be written 
	\begin{align*}
	f_\V(p) &= h(p) - \delta(p \mid C)\\
		 &= \min(\inprod{p}{b_1} - \beta_1, \ldots, \inprod{p}{b_k} - \beta_k) - \delta(p \mid \simplex)\\
		 &= \min(\inprod{p}{b_1 - \beta_1 \ones}, \ldots, \inprod{p}{b_k - \beta_k \ones}) - \delta(p \mid \simplex)~.~
	\end{align*}
\end{proof}

This allows us to project $g_\V$ from $\reals^\Y_+$ to the simplex $\simplex$.
\begin{lemma}\label{cor:f-matches-g-on-simplex}
  For all polyhedral concave $f_\V : \reals^\Y_+ \to \reals_+ \cup \{-\infty\}$ with $\dom(f_\V) = \simplex$, there is a polyhedral concave function $g_\V : \reals^\Y_+ \to \reals_+$ with $\dom(g_\V) = \reals^\Y_+$ so that $f_\V(p) = g_\V(p)$ for all $p \in \simplex$.
\end{lemma}
Given the function $f_\V$, we consider $g_\V$ to be its extension and $L$ such that $\risk L_+ = g_\V$ so that we may use the tools in \S~\ref{appsubsec:phase1} and~\ref{appsubsec:phase2} to draw conclusions about the property $\Gamma := \prop{L}$ defined on the simplex.

Define the function $\theta(v) = \{p \in \simplex \mid \inprod{v}{p} = f_\V(p)\}$ as the level sets of the loss vector $v \in \V$, and the set $\Theta := \{\theta(v) \mid v \in \V^*\}$ to be the set of (minimal) level sets uniquely defined by the loss vectors.

\begin{claim}\label{claim:theta-is-pi-cap-simplex}
  For all $v \in \V$, $\theta(v) = \pi(F_v) \cap \simplex$.
\end{claim}
\begin{proof}

  Fix $p \in \simplex$.
  \begin{align*}
    p \in \theta(v)
    &\iff \inprod{v}{p} = f_\V(p) & \text{Definition of $\theta$}\\
    &\iff \inprod{v}{p} = \min_{v' \in \V}\inprod{v'}{p} & \text{$f_\V=g_\V$ on $\simplex$ (Cor.~\ref{cor:f-matches-g-on-simplex})}\\
    &\iff v \in \argmin_{v' \in \V}\inprod{v'}{p} &\\
    &\iff p \in \pi(F_v) & \text{Lemma~\ref{lem:projected-faces-iff-support-iff-argmin}}
  \end{align*}
\end{proof}

Moreover, we can interchangably write the level set in terms of the loss's minimizing report or the generating loss vector.

\begin{claim}\label{claim:level-set-is-projected-face}
  For all $r \in \R$ with $v = L(r)$, $\Gamma_r = \theta(v) = \pi(F_{v}) \cap \simplex$.
\end{claim}
\begin{proof}
  Let us rewrite
  \begin{align*}
    \Gamma_r
    &= \{p \in \simplex \mid r \in \argmin_{r' \in \R} \inprod{L(r')}{p}\}\\
    &= \{p \in \simplex \mid v \in \argmin_{v' \in \V} \inprod{v'}{p}\}\\
    &= \{p \in \simplex \mid \inprod{v}{p} = \min_{v' \in \V}\inprod{v'}{p}\}\\
    &= \{p \in \simplex \mid \inprod{v}{p} = f(p) \}\\
    &= \theta(v)
  \end{align*}
  The rest of the result follows from Claim~\ref{claim:theta-is-pi-cap-simplex}.
   
\end{proof}

\subsection{Proving Lemma~\ref{lem:X}}\label{subsec:lem-X-proof}
Now that we have translated from $\reals^d_+$ to $\reals^\Y_+$ in \S~\ref{subsec:project-pi} and from $\reals^\Y_+$ to $\simplex$ in \S~\ref{subsec:project-f}, we can take the final steps to prove Lemma~\ref{lem:X}.
\begin{lemma}\label{lem:g-1-homog}
	Suppose we have the sets $\V$ and the finite set $\V^* \subseteq \V$ such that $g_\V = g_{\V^*}$.
	Then $g_{\V^*}(x) = \min_{v \in \V^*}\inprod{v}{x}$ is $1$-homogeneous.
\end{lemma}
\begin{proof}
	If $x \not \in \reals^\Y_+$, then $c g(x) = -\infty = g(cx)$ for any $c > 0$.
	If $x \in \reals^\Y_+$, then we have $g(cx) = \min_{v \in \V^*}\inprod{v}{cx} = c \min_{v \in \V^*}\inprod{v}{x} = c g(x)$ for any $c > 0$ by linearity of the inner product.
\end{proof}

Again we assume $L$ is minimizable, and $\risk L_+$ is polyhedral with $\V := L(\R)$, and $\Gamma := \prop{L}$.
We now define the \emph{extended level set} $\bar \Gamma_r := \{x \in \reals^\Y_+ \mid \inprod{L(r)}{x} = \risk L_+(x)\}$.

\begin{lemma}\label{lem:levelset-to-extended-levelset}
	For any $r \in \R$ and $c > 0$, if $p \in \Gamma_r$, then $cp \in \bar \Gamma_r$. 
\end{lemma}
\begin{proof}
	Fix $r \in \R$ and $c > 0$.
	We have
	\begin{align*}
	p \in \Gamma_r
	&= \{p' \in \simplex \mid r \in \argmin_{r' \in \R} \inprod{L(r')}{p'} \} & \text {Definition of level set} \\
	&= \{p' \in \simplex \mid v \in \argmin_{v' \in \V} \inprod{v'}{p'} \} & \text {$\V = L(\R)$} \\
	&= \{p' \in \simplex \mid \inprod{v}{p'} = \min_{v' \in \V} \inprod{v'}{p'} \} & \text {$L$ minnable} \\
	&= \{p' \in \simplex \mid \inprod{v}{p'} = g_\V(p') \} & \text {Definition of $g_\V$} \
	&= \{p' \in \simplex \mid c \inprod{v}{p'} = c g_\V(p') \} &  \\
	&= \{p' \in \simplex \mid  \inprod{v}{cp'} = g_\V(cp') \} & \text {Lemma~\ref{lem:g-1-homog}} \\
	\implies cp
	&\in \{x \in \reals^\Y_+ \mid \inprod{v}{x} = g_\V(x)\}\\
	&= \bar \Gamma_r
	\end{align*}
\end{proof}

\begin{lemma}\label{lem:extended-levelset-equals-projected-face}
	For any $r\in \R$ with $v = L(r)$, $\bar \Gamma_r = \pi(F_v)$.  
\end{lemma}
\begin{proof}
	\begin{align*}
	\bar \Gamma_r
	&= \{x \in \reals^\Y_+ \mid \inprod{L(r)}{x} = \risk L_+(x)\} & \text{Definition of $\bar \Gamma_r$}\\
	&= \{x \in \reals^\Y_+ \mid \inprod{L(r)}{x} = g_\V(x)\} & \text{Assumption that $\risk L_+(x) = g_\V(x)$}\\
	&= \{x \in \reals^\Y_+ \mid \inprod{v}{x} = g_\V(x)\} & \text{$v = L(r)$}\\
	&= \pi(F_v) & \text{Since $F_v = \{(x,g_\V(x)) \mid \inprod{v}{x} = g_\V(x)\}$}\\
	\end{align*}
\end{proof}

\begin{claim}\label{claim:projected-faces-cover-RY-iff-representative}
	A set $\R' \subseteq \R$ with $\V' := L(\R')$ is representative for $L$ if and only if $\cup_{v \in \V'} \pi(F_v) = \reals^\Y_+$.  
\end{claim}
\begin{proof}
	$\implies$
	This proof follows from three lemmas: first, we observe that $g$ is $1$-homogeneous (Lemma~\ref{lem:g-1-homog}).
	Then we extend the notion of a level set $\Gamma_r$ to the nonnegative orthant $\bar \Gamma_r$, and show that any scalar transformation of a distribution in the level set is contained in the same (extended) level set via Lemma~\ref{lem:levelset-to-extended-levelset}.
	Finally, we show the extended level set is exactly the projection $\pi(F_v)$ (Lemma~\ref{lem:extended-levelset-equals-projected-face}).
	As a corollary, we chain the results to observe $\cup_{r \in \R'} \Gamma_r = \simplex \implies \cup_{r \in \R'} \bar \Gamma_r = \reals_+^\Y = \cup_{v \in L(\R')}\pi(F_v) = \reals^\Y_+$, yielding the forward implication.

	$\impliedby$
	Fix $p \in \simplex \subseteq \reals^\Y_+$.
	By the assumption, there is a $v \in \V'$ such that $p \in \pi(F_v)$.
	By Claim~\ref{claim:level-set-is-projected-face}, we have $p \in \pi(F_v) \cap \simplex = \Gamma_r$ for the $r \in \R'$ such that $v = L(r)$.
	As this is true for all $p \in \simplex$, we have $\R'$ representative.

\end{proof}

\begin{lemma}\label{lem:lemX1-rep-iff-subset-vectors}
	A finite set $\R' \subseteq \R$ with $\V' = L(\R')$ is representative if and only if $\V^* \subseteq \V'$.
\end{lemma}
\begin{proof}
	Chain Claim~\ref{claim:projected-faces-cover-RY-iff-representative} and Claim~\ref{claim:Vprime-projected-faces-cover-iff-Vstar-subset-Vprime} to yield the result.
\end{proof}

\begin{lemma}\label{lem:lemX-3-rep-iff-FDLS-subsets}
	A finite set $\R' \subseteq \R$ with $\V' = L(\R')$ is representative if and only if $\Theta \subseteq \{\theta(v) \mid v \in \V'\}$.
\end{lemma}
\begin{proof}
	Chain Claim~\ref{claim:projected-faces-cover-RY-iff-representative} and Claim~\ref{claim:Vprime-projected-faces-cover-iffprojected-faces-subsets} to yield the result.
\end{proof}

Recall $\Theta := \{\theta(v) \mid v \in \V^*\}$; it follows that this set is exactly the set of level sets of the property elicited by $L$.
Moreover, let $\R^*$ be the finite set of reports given by Corollary~\ref{cor:finite-rep-set}.

\begin{corollary}
	$\Theta = \{\Gamma_r \mid r \in \R^*\}$ 
\end{corollary}

\begin{lemma}\label{lem:fdls-exactly-theta}
	$\Theta = \{\Gamma_r \mid r\in\R, \dim(\Gamma_r) = |\Y|-1\}$.
\end{lemma}
\begin{proof}
	From Claim~\ref{claim:projected-Vstar-faces-full-dim}, we know $\Lambda$ is exactly the set of full-dimensional level sets in $\reals^\Y_+$.
	Each element of $\Lambda$ is $\pi(F_v)$ for some $v \in \V^*$.
	Take $r \in \R^*$ so that $v = L(r)$.
	By Claim~\ref{claim:level-set-is-projected-face}, we have $\theta(v) = \Gamma_r = \pi(F_v) \cap \simplex$ is full-dimensional relative to the simplex.
	The result follows.
\end{proof}

\begin{lemma}\label{lem:any-levelset-contained-in-minlevelset}
  For any $r \in \R$, there exists a $v^* \in \V^* =: L(\R^*)$ such that $\Gamma_r \subseteq \theta(v^*)$.
\end{lemma}
\begin{proof}
  Take $v = L(r)$.
  By Corollary~\ref{cor:exists-vstar-projected-face-subset}, there is a $v^* \in \V^* \subseteq \V$ such that $\pi(F_v) \subseteq \pi(F_{v^*})$.
  Therefore, $\pi(F_v) \cap \simplex \subseteq \pi(F_{v^*})\cap \simplex$.  
  We know $\theta(v) = \pi(F_v) \cap \simplex$ and similarly for $\theta(v^*)$ by Claim~\ref{claim:level-set-is-projected-face}.
  The result follows.
\end{proof}

Now this brings us to Lemma~\ref{lem:X}.
The framework in this appendix cues up this proof: any loss $L$ satisfying the assumptions of Lemma~\ref{lem:X} has some $g_{L(\R)} = \risk{L}_+$ as in this section that we can work with.

\lemmaX*
\begin{proof}
Consider $\risk L = f_\V$ for a finite set $\V$ by Claim~\ref{claim:f-min-affine}.
There is a polyhedral concave function $g_\V$ on $\reals^\Y_+$ matching $f_\V$ on $\simplex$ by Corolary~\ref{cor:f-matches-g-on-simplex}.
Moreover, consider $g_\V = \risk L_+$, and observe that $\risk L_+$ matches $\risk L$ on $\simplex$ as well.
By Corollary~\ref{cor:unique-set-loss-vectors-defining-facets}, we then have a finite set $\V^*$ of smallest cardinality such that $f_\V = f_{\V^*}$ and $g_\V = g_{\V^*}$.
Consider $\R^* \subseteq \R$ such that $\V^* = L(\R^*)$.
First, observe that $\R^*$ is representative for $L$ as a corollary of Claim~\ref{claim:projected-faces-cover-RY-iff-representative}.
Moreover, consider the follow set of level sets of $f_{\V^*}$, $\Theta = \{\theta_v \mid v \in \V^*\}$.

Now that we have the preliminaries, consider the itemized statements.
For $f_{\V^*}$, Lemma~\ref{lem:lemX1-rep-iff-subset-vectors} is exactly statement \eqref{item:X-rep-V}.
This immediately implies statement~\eqref{item:X-min-V}.
Moreover, Lemma~\ref{lem:lemX-3-rep-iff-FDLS-subsets} is exactly statement~\eqref{item:X-rep-Theta}, and again statement~\eqref{item:X-min-Theta} immediately follows.
Statement~\eqref{item:X-rep-contain-min} is a corollary of the existence of a finite representative set, as shown in Corollary~\ref{cor:finite-rep-set}.
Again, Statement~\eqref{item:X-full-dim} is exactly Lemma~\ref{lem:fdls-exactly-theta}.
Statement~\eqref{item:X-redundant} is exactly Lemma~\ref{lem:any-levelset-contained-in-minlevelset}.
Finally, Statement~\eqref{item:X-tight-embed} follows as a corollary of statement~\eqref{item:X-min-V} and Corollary~\ref{cor:tight-embed-min-rep}.
\end{proof}

\end{document}